\documentclass[conference]{IEEEtran}
\IEEEoverridecommandlockouts
\usepackage{cite}
\usepackage{amsmath,amssymb,amsfonts}
\usepackage{amsthm}
\usepackage{algorithm2e}
\usepackage{graphicx}
\usepackage{textcomp}
\usepackage{xcolor}
\usepackage{float}
\usepackage[inkscapelatex=false]{svg}
\usepackage{stmaryrd}
\usepackage{hyperref}
\usepackage{multicol}
\usepackage{tabularx}
\usepackage{multirow}
\usepackage{soul}
\usepackage{xcolor}
\usepackage[normalem]{ulem}
\useunder{\uline}{\ul}{}
\usepackage{dblfloatfix} 
\sethlcolor{yellow!30}

\definecolor{replyblue}{RGB}{0,70,170}

\newcommand{\argmin}{\mathop{\mathrm{arg\,min\,}}}

\newtheorem{theorem}{\textbf{Theorem}}
\newtheorem{definition}{\textbf{Definition}}
\newtheorem{lemma}{\textbf{Lemma}}

\def\BibTeX{{\rm B\kern-.05em{\sc i\kern-.025em b}\kern-.08em
    T\kern-.1667em\lower.7ex\hbox{E}\kern-.125emX}}
    
\begin{document}

\title{A convergent Plug-and-Play Majorization-Minimization algorithm for linear Poisson inverse problems  \\
\thanks{This work was partially funded by the PRIMES LABEX (ANR-11-LABX-0063) at the University of Lyon, within the ”Investissements d’Avenir” program (ANR-11-IDEX-0007).}
}

\author{
\IEEEauthorblockN{Thibaut Modrzyk$^1$, Ane Etxebeste$^1$, Elie Bretin$^2$, Voichita Maxim$^1$}
\IEEEauthorblockA{$^1$ INSA‐Lyon, CREATIS UMR 5220 U1294, Villeurbanne, F-69621, France}
\IEEEauthorblockA{$^2$ INSA-Lyon, CNRS UMR 5208, Institut Camille
Jordan, Villeurbanne, F-69621, France.}
}

\maketitle

\begin{abstract}

In this paper, we present a novel variational plug-and-play algorithm for Poisson inverse problems. Our approach minimizes an explicit functional which is the sum of a Kullback-Leibler data fidelity term and a regularization term based on a pre-trained neural network. By combining classical likelihood maximization methods with recent advances in gradient-based denoisers, we allow the use of pre-trained Gaussian denoisers without sacrificing convergence guarantees. The algorithm is formulated in the majorization-minimization framework, which guarantees convergence to a stationary point. Numerical experiments confirm state-of-the-art performance in deconvolution and tomography under moderate noise, and demonstrate clear superiority in high-noise conditions, making this method particularly valuable for nuclear medicine applications.

\end{abstract}


\section{Introduction}
\label{sec:intro}

Linear Poisson inverse problems have been studied extensively over the past 50 years, due to their relevance to a wide range of applications in imaging \cite{berteroInverseImagingPoisson2018}. In this setting, the objective is to recover true data $x \in \mathbb{R}^{n}_+$ from its degraded version $y \in \mathbb{R}^{m}_+$, with a degradation modeled as:
\begin{equation}
    y \sim \frac{1}{\zeta}\mathcal{P}(\zeta Ax) ,
    \label{eq:inverse-pb}
\end{equation}
where $A$ denotes both a linear operator and its matrix with positive coefficients ($A_{ij} \geq 0$), which could represent a blur, a down-scaling, or a Radon projector, $\mathcal{P}$ represents a noise corruption, for instance following a Poisson distribution, and $\zeta$ is a gain parameter controlling the strength of the noise.

A standard approach to solve such an inverse problem is to formulate it as an optimization task, aiming at minimizing the negative log-likelihood (NLL) of the reconstruction with respect to the degraded data. However, due to the ill-posedness of the forward operator $A$, a regularization term is often added to constrain the space of plausible reconstructions.
This then amounts to solving the following optimization problem:
\begin{equation}
    x^* \in \argmin_x \left\{ f(x) + \lambda g(x) \right\}
    \label{eq:map}
\end{equation}
where $f$ is the NLL, and $g$ is the regularization imposed on the reconstruction $x$. 
For Poisson-distributed measurements, the data-fidelity term associated with the likelihood can be written, up to additive constants independent of $x$, as a generalized Kullback--Leibler (KL) divergence between the observations $y$ and the predicted measurements $Ax$.
Discarding these constants, which do not affect the optimization, yields the following objective:
\begin{equation}
    f(x) = \sum_{i=1}^m \left[ (Ax)_i - y_i \log (Ax)_i \right] ,
    \label{eq:poisson-likelihood}
\end{equation}
where $i$ indexes the measurement bins. By convention $\log(0) = - \infty$.
Poisson noise arises naturally in a variety of imaging applications, including astronomy \cite{berteroInverseImagingPoisson2018}, medical imaging modalities such as computed tomography (CT), positron emission tomography (PET), and single photon emission computed tomography (SPECT) \cite{sheppMaximumLikelihoodReconstruction1982a, langeEMReconstructionAlgorithms1984a}, as well as confocal microscopy \cite{deyRichardsonLucyAlgorithmTotal2006}.

In scenarios with high signal-to-noise ratio (SNR), the Poisson noise model can be effectively approximated by a Gaussian distribution using the Anscombe transform \cite{anscombeTransformationPoissonBinomial1948}. However, this approximation fails in low-SNR settings, which are characteristic of most of the aforementioned imaging applications. In this setting, the Poisson data-fidelity term (\ref{eq:poisson-likelihood}), while convex, exhibits a non-Lipschitz gradient near zero, and its proximal operator does not admit a closed-form expression when $A \neq \text{Id}$. This means that a straightforward application of most well-known optimization schemes is impossible. This difficulty has motivated the development of specialized optimization schemes tailored to Poisson data-fidelity terms, including Expectation Maximization (EM) algorithms \cite{sheppMaximumLikelihoodReconstruction1982a}, scaled gradient projections \cite{bonettiniNewConvergenceResults2015}, augmented Lagrangian methods \cite{figueiredoRestorationPoissonianImages2010}, and Bregman proximal gradient algorithms based on Burg’s entropy \cite{bauschkeDescentLemmaLipschitz2017}.

An additional challenge is the selection of an appropriate regularization $g$ tailored to the specific application. Among the most widely used priors is Total Variation (TV) regularization \cite{rudinNonlinearTotalVariation1992a}, which promotes piecewise-constant solutions, facilitating the recovery of sharp edges in reconstructions. In the case of hand-crafted priors, several algorithms \cite{figueiredoRestorationPoissonianImages2010, chambolleFirstOrderPrimalDualAlgorithm2011a, maximTomographicReconstructionPoisson2023} provide convergence guarantees in the case of a convex regularization.

The progress of deep learning across various domains of image processing has shifted the focus from traditional hand-crafted priors to learned regularizations, typically implemented with neural networks. Among the methods that integrate deep learning, one can mention the Plug-and-Play (PnP) \cite{venkatakrishnanPlugPlayPriorsModel2013} and the Regularization by Denoising (RED) \cite{TheLittle} frameworks.
However, their convergence is often constrained by strong and sometimes unrealistic assumptions on the learned model, such as requiring the neural network to be convex or non-expansive, even in the relatively simpler case of Gaussian noise \cite{ryuPlugandPlayMethodsProvably2019a, chanPlugandPlayADMMImage2017}. These assumptions are typically enforced through specific architectural choices for the network \cite{terrisBuildingFirmlyNonexpansive2020}, which can limit the performance compared to state-of-the-art denoising architectures, such as DRUNet \cite{zhangPlugPlayImageRestoration2022} and DnCNN \cite{zhangGaussianDenoiserResidual2017a}.

Expressing the learned priors as explicit gradient-step or proximal operators enable their use without constraining the expressivity of the neural network \cite{huraultGradientStepDenoiser2021, cohenItHasPotential2021, huraultProximalDenoiserConvergent2022a}. These \textit{Gradient-Step} (GS) denoisers can then be seamlessly integrated into various iterative splitting schemes, such as Forward-Backward splitting, the Alternating Direction Method of Multipliers (ADMM), and Douglas-Rachford splitting, while retaining provable convergence guarantees.

While originally intended for Gaussian noise, these GS-PnP algorithms can also handle data-fidelity terms with non-Lipschitz gradient such as the Poisson NLL. However, such algorithms rely on the computation of the proximal operator of $f$. When this proximal operator does not admit a closed-form solution, we have to resort to iterative approximations. This typically introduces an inner optimization loop, increasing the overall computational cost, and, unless the approximation errors are carefully controlled (e.g., by running the inner solver to sufficient accuracy), such inexact proximal steps may compromise the theoretical convergence guarantees of the algorithm \cite{NIPS2011_8f7d807e}.

We present a novel convergent PnP algorithm adapted to Poisson noise where we incorporate ideas from Majorization-Minimization (MM) algorithms. Building on the classical approaches to Poisson inverse problems, we arrive at an algorithm closely related to well-known multiplicative update schemes derived from the EM algorithm \cite{dempsterMaximumLikelihoodIncomplete1977}. Our method splits the optimization of the data-fidelity term and the regularization while maintaining closed-form updates at each iteration. It leverages advancements from both the deep learning literature, by incorporating regularization through a neural network, and the classical optimization and image processing domains, through the selection of an appropriate domain-specific surrogate function. 
Notably, our algorithm retains theoretical convergence guarantees on Poisson inverse problems while still relying on a widely adopted and easy to train Gaussian denoiser. The final algorithm also preserves the non-negativity constraint at each iteration.

The main contributions of this paper are the following:
\begin{itemize}
    \item We propose a new PnP algorithm for solving Poisson inverse problems, using an appropriate surrogate function of the KL divergence.
    \item We show that each iteration decreases an explicit objective function, ensuring convergence to a stationary point of this objective.
    \item The Gaussian GS denoiser satisfies the conditions required for theoretical convergence when combined with the Poisson negative log-likelihood, while also enabling strong empirical performance.
    \item We extensively evaluate the method on several linear inverse problems and datasets, including PET reconstruction and CT imaging corrupted by Poisson and Poisson–Gaussian noise.
    \item Finally, to facilitate future comparisons, we provide an open-source PyTorch implementation of our method, relying on the DeepInverse python library \cite{Tachella_DeepInverse_A_deep_2023}. This implementation will be made available at \href{https://github.com/Tmodrzyk/PnP-MM}{https://github.com/Tmodrzyk/PnP-MM} upon acceptance of this manuscript. 
\end{itemize}

 In section \ref{sec:methods} we introduce the methods and notations relative to Poisson inverse problems, along with PnP methods. In section \ref{sec:pnp-mm} we present in detail our method along with an analysis of its convergence properties. Finally, in section \ref{sec:exp} we present numerical results on several Poisson inverse problems: deconvolution, and tomographic reconstruction in various settings. 
    
\section{Related works}
\label{sec:methods}

 In this section, we provide a comprehensive overview of the concepts our method rely on, highlighting their strengths, limitations, and similarities. We also introduce the PnP framework, which forms the foundation of our work and is extended here to the Poisson data distribution.

\subsection{Majorization-Minimization algorithms}
\label{subsec:methods:mm}

One of the earliest and most extensively studied Poisson inverse problems is deconvolution, where the linear operator $A$ represents a convolution by a known Point-Spread-Function kernel. The Richardson-Lucy algorithm \cite{richardsonBayesianBasedIterativeMethod1972a ,lucyIterativeTechniqueRectification1974a} employs a Bayesian perspective to formulate an iterative scheme that minimizes the Poisson NLL cost function described in (\ref{eq:poisson-likelihood}). 

Later interpreted in the Expectation Maximization (EM) framework \cite{dempsterMaximumLikelihoodIncomplete1977}, it found widespread application in the field of PET through the work of Shepp and Vardi \cite{sheppMaximumLikelihoodReconstruction1982a}. The introduction of MLEM inspired the development of numerous other EM-based algorithms for emission tomography \cite{langeEMReconstructionAlgorithms1984a, hudsonAcceleratedImageReconstruction1994, pierroModifiedExpectationMaximization1995, depierroFastEMlikeMethods2001}, several of which are implemented in PET scanners nowadays.

In the early 2000s, the algorithm emerged again under a different name in the Non-Negative Matrix Factorization (NMF) community \cite{leeAlgorithmsNonnegativeMatrixa}. A key difference with previous studies is that the NMF setting would be equivalent to the blind inverse problem, where the operator $A$ is estimated during the reconstruction. 
For a comprehensive overview of these algorithms in the context of NMF, interested readers can refer to \cite{fevotteAlgorithmsNonnegativeMatrix2011}.

A notable progress made by the tomography and NMF communities is the interpretation of these iterations as a MM algorithm. The general MM framework \cite{langeMMOptimizationAlgorithms2016}, including for instance EM methods as well as some proximal splitting algorithms, relies heavily on finding for $f$ an appropriate surrogate function $F$ called a \textit{tangent majorant} \cite{langeMMOptimizationAlgorithms2016}, which serves as an upper bound to the original objective function within the MM framework.
\begin{definition}
    \label{def:surrogate}
    Let $f : \mathbb{R}^d \rightarrow \mathbb{R}$ be a differentiable function. A function $F(x, \tilde{x})$ is called a tangent majorant of $f$ if it satisfies the following two conditions:
    \begin{align*}
        \forall (x, \tilde{x}) \in (\mathbb{R}^d)^2, \quad f(x) &\leq F(x, \tilde{x})
        \\ \forall x \in \mathbb{R}^d, \quad f(x) &= F(x, x).
    \end{align*}
\end{definition}

The iterations constructed from this surrogate function, namely:
\begin{equation}
    x^{(n+1)} = \argmin_x F(x, x^{(n)}),
    \label{eq:mm}
\end{equation}
can be proven to produce a sequence $(x^n)$ non-increasing with respect to the original data-fidelity function $f$. 
In the case of Poisson inverse problems, the surrogate function is often chosen as:
\begin{equation}
    F_{\text{EM}}(x, \tilde{x}) = \sum_i \left[ (Ax)_i - \sum_j y_i \tilde{\alpha}_{ij} \log \left( \frac{A_{ij} x_j}{\tilde{\alpha}_{ij} } \right) \right]
    \label{eq:surrogate}
\end{equation}
where $\tilde{\alpha}_{ij} = \frac{A_{ij} \tilde{x}_j}{(A\tilde{x})_i}$.
This majorant is notably employed in algorithms such as MLEM \cite{sheppMaximumLikelihoodReconstruction1982a} and KL-MU \cite{leeAlgorithmsNonnegativeMatrixa}.
For more details on the derivations and origins of this function, we refer readers to \cite{fevotteAlgorithmsNonnegativeMatrix2011}. 
Minimizing this tangent majorant leads to the well-known multiplicative updates:
\begin{equation}
    x^{(n+1)} = \frac{x^{(n)}}{s} \cdot A^\top \frac{y}{A x^{(n)}}
    \label{eq:rl-iter}
\end{equation}
where the multiplication $\cdot$ and division are element-wise, and $s = A^\top \mathbf{1}$ is a normalization term often called \textit{sensitivity} of the forward operator.
For completeness, the exact derivations leading to these updates can be found in Appendix \ref{sec:proof:lemma-mu}. 

\subsection{Plug-and-Play framework}
\label{subsec:methods:pnp}

PnP methods \cite{venkatakrishnanPlugPlayPriorsModel2013} replace the proximal operator associated with the regularization term $g$ by a denoiser $D_\theta$. Several proximal splitting algorithms, such as Half-Quadratic Splitting and ADMM, have been adapted to the PnP framework \cite{chanPlugandPlayADMMImage2017}. 

The recent work on GS denoisers \cite{huraultGradientStepDenoiser2021, cohenItHasPotential2021, huraultProximalDenoiserConvergent2022a} formulates the following explicit regularization:
\begin{equation}
    g_\sigma(x) = \|x - N_\sigma(x)\|^2,
    \label{eq:gradient-step}
\end{equation}
where $N_\sigma$ can be any differentiable neural network, taking the noise level $\sigma$ as an additional input. 
The denoiser is then formulated as an explicit gradient step over this regularization:
\begin{equation}
    D_\sigma (x) = x - \tau \nabla g_\sigma (x).
    \label{eq:gradient-step-denoiser}
\end{equation}
where $\tau$ is the gradient step-size. 
Hurault \textit{et al.} \cite{huraultGradientStepDenoiser2021} train this denoiser using the standard denoising score-matching loss and demonstrate that the resulting model achieves performance comparable to other state-of-the-art denoisers.
The gradient step denoiser allows to construct a convergent PnP forward-backward algorithm (GS-PnP-PGD) whose iterations are: 
\begin{equation}
    x^{(n+1)} = \text{Prox}_{\tau f} \circ (\text{Id} - \tau \lambda \nabla g_\sigma) (x^{(n)}).
    \label{eq:gs-pnp}
\end{equation}

Hurault \textit{et al}. \cite{huraultConvergentBregmanPlugandplay2024a} extended the GS-PnP framework to the Bregman geometry, enabling the handling of Poisson inverse problems with a denoiser specifically trained on noise related to inverse-gamma distributions. 

We depart from the Bregman geometry paradigm and instead use a Gaussian GS denoiser within the MM framework. The resulting iterates bear similarities to the approach proposed by Rond \textit{et al}. \cite{rondPoissonInverseProblems2016}. However, their method relies on an extension of ADMM rather than MM, and needs an inner optimization loop to approximate the proximal operator of $f$, limiting convergence guarantees. Another recent approach combined RL iterates with a diffusion model regularization, which proved efficient for deconvolution, but again without any convergence guarantees \cite{modrzykFastDeconvolutionUsing2024}.

\begin{figure*}[t]
    \centering
    \includegraphics[width=1\linewidth]{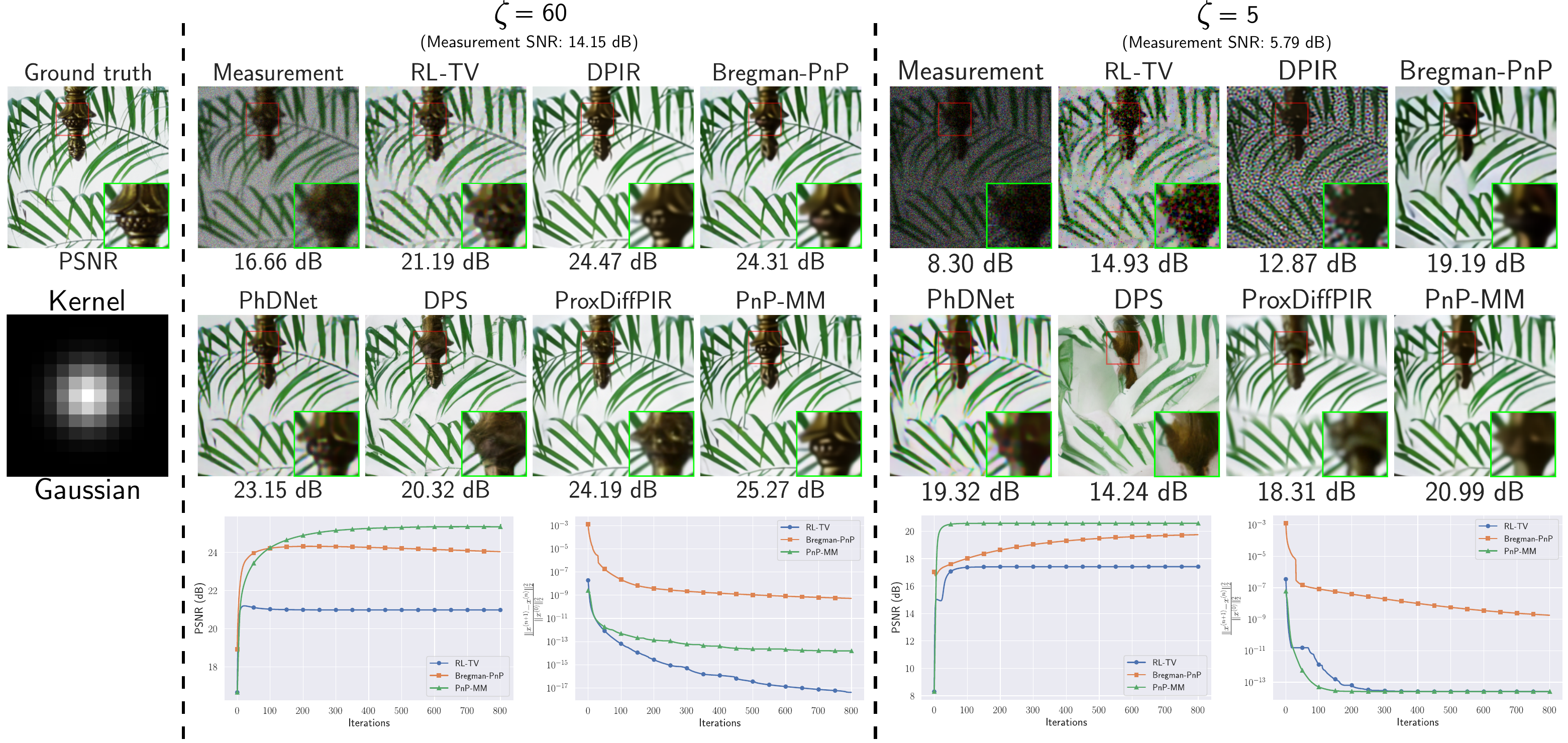} 
    \caption{Left: Deconvolution on the \textit{Leaves} image from the Set3c dataset with the Gaussian kernel and noise level $\zeta=60$. Right: Same image and kernel, with noise level $\zeta=5$. All images were produced from the same measurement data. Evolution of the PSNR and the gap between iterates have been plotted for theoretically convergent methods.}
    \label{fig:figure1}
\end{figure*}

\section{Plug-and-Play Majorization Minimization}
\label{sec:pnp-mm}

Although the standard FB algorithm for solving the regularized problem (\ref{eq:map}) can in itself be seen as an MM algorithm, here we adopt a different perspective by considering the scenario in which $\text{Prox}_{\tau f}$ is intractable.
In such cases, one may replace $f$ with a suitable tangent-majorant $F$ and proceed with updates very similar to the original FB algorithm.
The resulting algorithm, which we call \textit{Majorized Forward-Backward}, can then be described by the following iterations:
\begin{equation}
    x^{(n+1)} = \text{Prox}_{\tau F(\cdot, x^{(n)})} \circ (\text{Id} - \tau \lambda \nabla g)(x^{(n)}).
    \label{eq:mm-pgd}
\end{equation}
This general framework enables to replace the objective function $f$ with one's favorite surrogate.
We prove in Appendix \ref{sec:proof:pnp-convergence} that the algorithm has similar convergence properties as the original forward-backward splitting, as established in the following theorem.
\begin{theorem}
    Let $f : \mathbb{R}^d \rightarrow \mathbb{R} \cup \{+\infty\}$ and $g : \mathbb{R}^d \rightarrow \mathbb{R}$ be proper lower semicontinuous functions, with $f$ convex and $g$ differentiable with an $L$-Lipschitz gradient, and $\lambda > 0$ with $\tau < 1 / \lambda L$. Assume that $h= f+\lambda g$ is bounded from below. Let $F: \mathbb{R}^{2d} \rightarrow \mathbb{R} \cup \{+\infty\}$ be a proper lower semicontinuous tangent majorant of $f$, convex with respect to its first argument. Then:
    \begin{enumerate}
        \item The sequence $\left( h(x^{(n)}) \right)$, where $\left( x^{(n)} \right)$ is generated by (\ref{eq:mm-pgd}), is non-increasing and convergent.
        \item The residual $\min_{0 \leq n \leq N} \|x^{(n+1)} - x^{(n)}\|^2$ converges to 0 with a rate of $\mathcal{O}\left(\frac{1}{N}\right)$.
        \item All cluster points of $(x^{(n)})$ are stationary points of $h$.
    \end{enumerate}
    \label{th:convergence-pnp}
\end{theorem}
This theorem allows the use of a wide range of regularizations $g$, including non-convex $L$-smooth ones. 

We will now specialize this general algorithm to the case of Poisson noise, with $f$ taking the expression described in (\ref{eq:poisson-likelihood}) and the canonical surrogate for the NLL defined in (\ref{eq:surrogate}). 
The proximal map of $F_{\text{EM}}$ admits a closed-form solution, whose computation is detailed in Appendix \ref{annex:prox-computation}. This means we can derive closed-form expressions for each minimization step (\ref{eq:mm-pgd}).
This leads to the following PnP-MM algorithm, where the first denoising step corresponds to a gradient update on the regularization $g_\sigma$ and the second step computes the proximal map of $F_{\text{EM}}(\cdot , x^{(n)})$:
\begin{equation*}
    \left\{
    \begin{array}{ll}
        x^{\left(n+\frac{1}{2}\right)} = \lambda \tau D_\sigma(x^{(n)}) + (1 - \lambda \tau) x^{(n)} \medskip \\ 
        x^{(n+1)} = \frac{1}{2} \left[ x^{\left(n+\frac{1}{2}\right)} - \tau s + \sqrt{\left( x^{\left(n+\frac{1}{2}\right)} - \tau s \right)^2 + 4 \tau s x_{\text{EM}}} \right]
    \end{array}
    \right.
    \label{eq:iterations}
\end{equation*}
where $x_{\text{EM}} = \displaystyle \frac{x^{\left(n+\frac{1}{2}\right)}}{s} \cdot A^\top \frac{y}{A x^{\left(n+\frac{1}{2}\right)}}$, which amounts to the regular MLEM update, and where $x^{(0)}$ is generally taken as a vector of ones.
A key property of this algorithm is that the iterates are given by taking the positive solution of a quadratic polynomial, and thus they respect the positivity constraint $x^{(n)} \geq 0$. Note that the convergence of the algorithm does not rely on the classical normalization assumption $A^\top 1 = 1$ often adopted in imaging problems.
Additionally, we apply a final denoising step similar to Hurault et \textit{al.} \cite{huraultGradientStepDenoiser2021}, to reduce residual noise.

\begin{table*}[t]
\caption{Comparison on the Kodak test dataset}
\label{tab:deblur-kodak-quantitative}
\resizebox{\textwidth}{!}{%
\label{tab:my-table}
\begin{tabular}{l|c|ccc|ccc|ccc|ccc|ccc}
\cline{2-14}
\multicolumn{1}{c|}{}         & Kernels                                                & \multicolumn{3}{c|}{Gaussian}                          & \multicolumn{3}{c|}{Levin1}                            & \multicolumn{3}{c|}{Levin4}                            & \multicolumn{3}{c|}{Uniform}                           & \multicolumn{3}{c}{Average}                            \\ \cline{1-14}
\multicolumn{1}{|c|}{$\zeta$} & Methods                                                & PSNR $\uparrow$ & SSIM $\uparrow$ & LPIPS $\downarrow$ & PSNR $\uparrow$ & SSIM $\uparrow$ & LPIPS $\downarrow$ & PSNR $\uparrow$ & SSIM $\uparrow$ & LPIPS $\downarrow$ & PSNR $\uparrow$ & SSIM $\uparrow$ & LPIPS $\downarrow$ & PSNR $\uparrow$ & SSIM $\uparrow$ & LPIPS $\downarrow$ \\ \cline{1-14}
\multicolumn{1}{|l|}{60}      & RL-TV \cite{maximTomographicReconstructionPoisson2023} & 24.78           & 0.610           & 0.425              & 23.33           & 0.520           & 0.493              & 24.00           & 0.544           & 0.440              & 23.80           & 0.568           & 0.474              & 23.98           & 0.561           & 0.458              \\
\multicolumn{1}{|l|}{}        & DPS \cite{chungDiffusionPosteriorSampling2022}         & 20.41           & 0.469           & 0.549              & 21.98           & 0.545           & 0.382              & 22.85           & 0.580           & 0.380              & 18.56           & 0.404           & 0.567              & 20.95           & 0.499           & 0.469              \\
\multicolumn{1}{|l|}{}        & Prox-DiffPIR \cite{melidonis2025scorebased}            & 26.04           & 0.688           & \textbf{0.337}     & 25.31           & 0.632           & \textbf{0.307}     & 26.31           & 0.676           & \textbf{0.255}     & 24.97           & {\ul 0.630}     & \textbf{0.391}     & 25.66           & 0.656           & \textbf{0.322}     \\
\multicolumn{1}{|l|}{}        & PhD-Net \cite{sanghviPhotonLimitedNonBlind2022}        & 26.00           & 0.685           & 0.420              & 24.56           & 0.608           & 0.424              & 25.81           & 0.665           & 0.424              & 24.70           & 0.618           & 0.503              & 25.26           & 0.644           & 0.455              \\
\multicolumn{1}{|l|}{}        & DPIR \cite{zhangPlugPlayImageRestoration2022}          & {\ul 26.46}     & {\ul 0.701}     & 0.403              & {\ul 26.00}     & 0.663           & 0.420              & {\ul 27.03}     & 0.711           & 0.350              & {\ul 24.78}     & 0.618           & 0.515              & {\ul 26.07}     & {\ul 0.673}     & 0.422              \\
\multicolumn{1}{|l|}{}        & B-PnP \cite{huraultConvergentBregmanPlugandplay2024a}  & 26.21           & 0.695           & 0.429              & 25.50           & {\ul 0.657}     & 0.427              & 26.93           & {\ul 0.716}     & 0.360              & 24.73           & 0.618           & 0.520              & 25.84           & 0.672           & 0.434              \\
\multicolumn{1}{|l|}{}        & \textbf{PnP-MM (Ours)}                                 & \textbf{26.69}  & \textbf{0.715}  & {\ul 0.382}        & \textbf{26.28}  & \textbf{0.748}  & {\ul 0.364}        & \textbf{27.64}  & \textbf{0.748}  & {\ul 0.300}        & \textbf{25.26}  & \textbf{0.640}  & {\ul 0.470}        & \textbf{26.47}  & \textbf{0.699}  & {\ul 0.379}        \\ \cline{1-14}
\multicolumn{1}{|l|}{20}      & RL-TV \cite{maximTomographicReconstructionPoisson2023} & 24.14           & 0.592           & 0.461              & 23.15           & 0.544           & 0.502              & 24.09           & 0.584           & 0.460              & 23.16           & 0.544           & {\ul 0.513}        & 23.64           & 0.566           & 0.484              \\
\multicolumn{1}{|l|}{}        & DPS \cite{chungDiffusionPosteriorSampling2022}         & 19.24           & 0.426           & 0.599              & 21.85           & 0.533           & \textbf{0.385}     & 23.13           & 0.576           & {\ul 0.363}        & 17.12           & 0.335           & 0.640              & 20.33           & 0.468           & 0.497              \\
\multicolumn{1}{|l|}{}        & Prox-DiffPIR \cite{melidonis2025scorebased}            & 25.11           & 0.642           & \textbf{0.382}     & 24.25           & 0.592           & {\ul 0.403}        & 25.17           & 0.632           & \textbf{0.347}     & 24.04           & 0.587           & 0.450              & 24.64           & {\ul 0.614}     & \textbf{0.396}     \\
\multicolumn{1}{|l|}{}        & PhD-Net \cite{sanghviPhotonLimitedNonBlind2022}        & 25.08           & {\ul 0.645}     & 0.468              & 23.80           & 0.576           & 0.539              & 24.84           & 0.624           & 0.480              & {\ul 23.96}     & 0.587           & 0.543              & 24.42           & 0.608           & 0.507              \\
\multicolumn{1}{|l|}{}        & DPIR \cite{zhangPlugPlayImageRestoration2022}          & {\ul 25.45}     & 0.605           & 0.458              & {\ul 24.49}     & 0.605           & 0.517              & {\ul 25.83}     & {\ul 0.656}     & 0.433              & 23.89           & 0.578           & 0.575              & {\ul 24.92}     & 0.623           & 0.496              \\
\multicolumn{1}{|l|}{}        & B-PnP \cite{huraultConvergentBregmanPlugandplay2024a}  & 24.45           & 0.614           & 0.589              & 24.18           & {\ul 0.609}     & 0.471              & 24.53           & 0.616           & 0.509              & 23.94           & {\ul 0.591}     & 0.585              & 24.28           & 0.608           & 0.538              \\
\multicolumn{1}{|l|}{}        & \textbf{PnP-MM (Ours)}                                 & \textbf{25.80}  & \textbf{0.676}  & {\ul 0.423}        & \textbf{25.12}  & \textbf{0.639}  & 0.430              & \textbf{26.29}  & \textbf{0.693}  & 0.366              & \textbf{24.54}  & \textbf{0.610}  & \textbf{0.510}     & \textbf{25.44}  & \textbf{0.654}  & {\ul 0.432}        \\ \cline{1-14}
\multicolumn{1}{|l|}{5}       & RL-TV \cite{maximTomographicReconstructionPoisson2023} & 22.15           & 0.482           & 0.591              & 21.59           & 0.454           & 0.598              & 21.52           & 0.438           & 0.614              & 22.10           & 0.497           & 0.611              & 21.84           & 0.468           & 0.603              \\
\multicolumn{1}{|l|}{}        & DPS \cite{chungDiffusionPosteriorSampling2022}         & 19.78           & 0.478           & 0.494              & 18.66           & 0.439           & 0.556              & 19.17           & 0.464           & 0.561              & 18.61           & 0.432           & \textbf{0.536}     & 19.06           & 0.453           & 0.537              \\
\multicolumn{1}{|l|}{}        & Prox-DiffPIR \cite{melidonis2025scorebased}            & 23.70           & {\ul 0.579}     & \textbf{0.468}     & {\ul 22.90}     & 0.534           & \textbf{0.512}     & {\ul 23.65}     & 0.573           & {\ul 0.470}        & {\ul 22.90}     & 0.531           & {\ul 0.538}        & {\ul 23.29}     & 0.554           & \textbf{0.497}     \\
\multicolumn{1}{|l|}{}        & PhD-Net \cite{sanghviPhotonLimitedNonBlind2022}        & {\ul 23.77}     & 0.592           & 0.534              & 22.73           & 0.536           & 0.597              & 23.43           & 0.568           & 0.551              & 22.86           & {\ul 0.546}     & 0.609              & 23.20           & 0.560           & 0.573              \\
\multicolumn{1}{|l|}{}        & DPIR \cite{zhangPlugPlayImageRestoration2022}          & 18.07           & 0.340           & 0.774              & 13.06           & 0.260           & 0.742              & 11.17           & 0.206           & 0.811              & 20.49           & 0.410           & 0.791              & 15.70           & 0.304           & 0.779              \\
\multicolumn{1}{|l|}{}        & B-PnP \cite{huraultConvergentBregmanPlugandplay2024a}  & 21.93           & 0.591           & 0.549              & 21.24           & {\ul 0.538}     & 0.599              & 21.98           & {\ul 0.591}     & 0.533              & 21.26           & 0.537           & 0.627              & 21.60           & {\ul 0.564}     & 0.577              \\
\multicolumn{1}{|l|}{}        & \textbf{PnP-MM (Ours)}                                 & \textbf{24.55}  & \textbf{0.623}  & {\ul 0.484}        & \textbf{23.51}  & \textbf{0.568}  & {\ul 0.536}        & \textbf{24.65}  & \textbf{0.625}  & \textbf{0.460}     & \textbf{23.49}  & \textbf{0.565}  & 0.577              & \textbf{24.05}  & \textbf{0.595}  & {\ul 0.514}        \\ \cline{1-14}
\end{tabular}
}
\end{table*}

In several imaging modalities, the measurement process is affected by both signal-dependent Poisson noise and additive electronic noise. An example of such an application is low-dose CT, where the photon statistics of the X-ray source induce Poisson fluctuations, while the acquisition chain contributes an additional approximately Gaussian component \cite{nuytsModellingPhysicsIterative2013}. In such cases, the observations are commonly expressed as a compound of Poisson and Gaussian noises where $z = y + \epsilon$ with $y \sim \frac{1}{\zeta}\mathcal{P}(\zeta Ax)$ and $\epsilon \sim \mathcal{N}(0, \sigma^2)$, where $y$ and $\epsilon$ are independent observations.
The likelihood of this Poisson-Gaussian distribution is not tractable. A widely adopted strategy is therefore to approximate the Poisson–Gaussian distribution by a Shifted Poisson (SP) model by defining $\hat{y} = z + \sigma^{2}$. The transformed measurement then admits the following approximation:
\begin{equation}
\hat{y} \sim \mathcal{P}(Ax +\sigma^{2}),
\label{eq:shifted-poisson}
\end{equation}
which yields a likelihood of the standard Poisson form up to a constant shift. Replacing the Poisson NLL by its SP counterpart results only in a small change of the surrogate for the data-fidelity term, and the update step becomes:
\begin{equation}
x_{\text{EM}} = \frac{x^{\left(n+\frac{1}{2}\right)}}{s}\cdot A^\top \frac{\hat{y}}{A x^{\left(n+\frac{1}{2}\right)} + \sigma^{2}},
\label{algo:poisson-gaussian}
\end{equation}
while the remaining steps of the algorithm remain unchanged. Because these iterates approximate $Ax + \sigma^2$, we subtract $\sigma^2$ to the final update of the PnP-MM algorithm.
The SP approximation requires $\hat{y}\geq 0$, a condition generally satisfied for realistic choices of $\zeta$ and $\sigma$. To strictly enforce this constraint in all cases, we set negative entries of $\hat{y}$ to zero before applying the algorithm.

\section{Experiments}
\label{sec:exp}
In this section, we assess the performance and versatility of PnP-MM across several imaging modalities and inverse problems.
We first examine Poisson deblurring with a range of blur kernels, using a Gaussian denoiser pre-trained on natural images. We then evaluate PnP-MM on PET reconstruction, retraining a denoiser on a small dataset of only ten volumes. Finally, we investigate the case of the mixed Poisson-Gaussian noise in a sparse-view CT reconstruction experiment, with the denoiser trained on a dataset with only five patients.

\subsection{Deblurring}
\label{subsec:exp:deblurring}

\subsubsection{Datasets}

We used the pre-trained GS denoiser proposed by Hurault \textit{et al.} \cite{huraultGradientStepDenoiser2021} and the training dataset introduced by Zhang \textit{et al.} in the DPIR paper \cite{zhangPlugPlayImageRestoration2022}, which is a combination of 400 images from the Berkeley Segmentation Dataset (CBSD) \cite{martinDatabaseHumanSegmented2001}, 4,744 images from the Waterloo Exploration Database \cite{maWaterlooExplorationDatabase2017}, 900 images from the DIV2K dataset \cite{agustssonNTIRE2017Challenge2017}, and 2,750 images from the Flickr2K dataset \cite{limEnhancedDeepResidual2017}.

We optimized the hyper-parameters of our method on the CBSD68 dataset \cite{huraultGradientStepDenoiser2021}, which serves as a validation set. The tests were carried out on two independent datasets to avoid hyper-parameter overfitting. The first testing dataset is the Kodak dataset \cite{franzenKodak}, consisting of 24 images, while the second is the well-known Set3C dataset. All the validation and testing images were center-cropped to the size $256\times256$. 

\subsubsection{Experimental setting}

We followed the same experimental setting as the Bregman-PnP method \cite{huraultConvergentBregmanPlugandplay2024a}. We used 4 blur kernels,  namely  a $25 \times 25$ Gaussian blur kernel with a standard deviation of 1.6, a uniform blur kernel of size $9 \times 9$, and 2 real-world motion blur kernels taken from \cite{levinUnderstandingEvaluatingBlind2009a}. All kernels can be found in Annex \ref{annex:tab:deblur-cbsd68-quantitative}. We considered four different noise levels $\zeta=5,20, 40, 60$ and we simulated Poisson data $y\sim \frac{1}{\zeta}\mathcal{P}(\zeta Ax)$. Note that the lowest SNR corresponds to $\zeta =5$ and the highest to $\zeta = 60$. These experiments aim to investigate the robustness of the selected methods to varying amounts of noise. 

\subsubsection{Methods}

We compare our method against representative state-of-the-art approaches for Poisson deconvolution. As classical baselines, we include TV-regularized Richardson–Lucy (RL–TV), implemented via a dual formulation tailored to Poisson noise \cite{maximTomographicReconstructionPoisson2023}, and DPIR, a Plug-and-Play half-quadratic splitting method originally designed for Gaussian noise \cite{zhangPlugPlayImageRestoration2022}. We further evaluate an unrolled three-operator splitting network specifically developed for photon-limited deblurring \cite{sanghviPhotonLimitedNonBlind2022}, as well as the convergent Bregman-PnP method of Hurault \textit{et al.}, which extends the Bregman Proximal Gradient descent \cite{bauschkeDescentLemmaLipschitz2017} to the Plug-and-Play setting \cite{huraultConvergentBregmanPlugandplay2024a}.
In addition, we consider two diffusion-based approaches: Diffusion Posterior Sampling (DPS) \cite{chungDiffusionPosteriorSampling2022}, implemented with the Poisson-stabilization strategy of \cite{melidonis2025scorebased}, and the score-based method of Melidonis \textit{et al.} \cite{melidonis2025scorebased}, which replaces the DPS likelihood gradient with a proximal update adapted to Poisson statistics.

\begin{figure*}[h]
    \centering
    \includegraphics[width=1\linewidth]{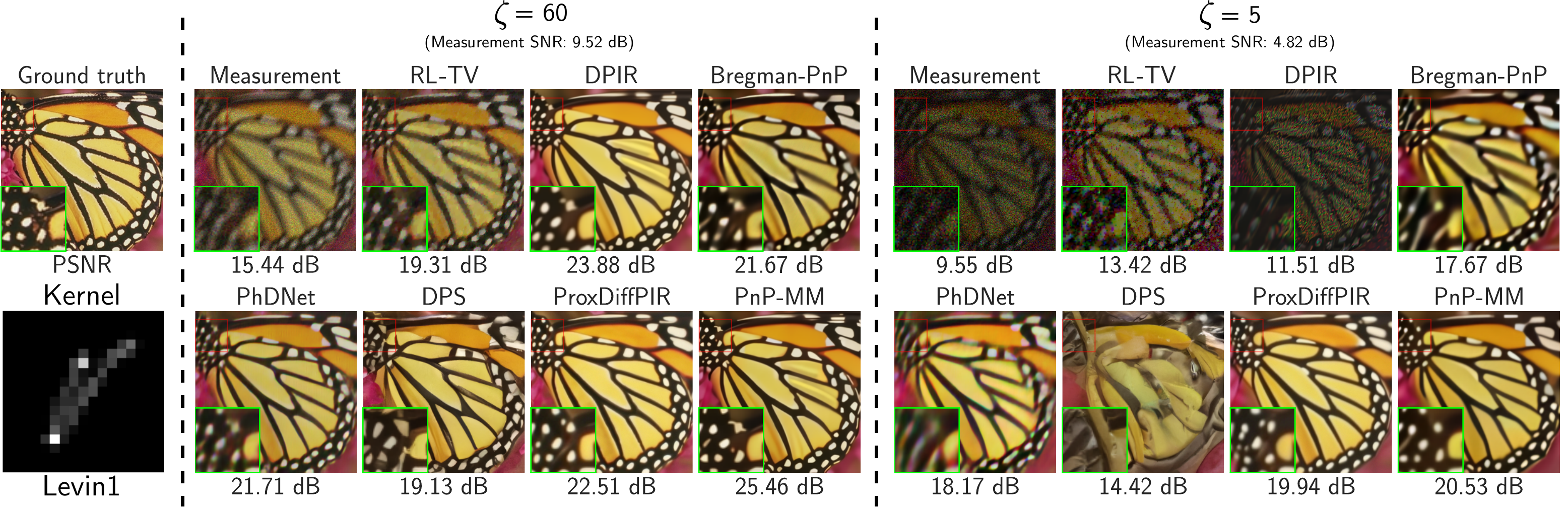} 
    \caption{Left: Deconvolution on the \textit{Butterfly} image from the Set3c dataset with the first Levin real world motion blur kernel and noise level $\zeta=60$. Right: Same image and kernel, with noise level $\zeta=5$.}
    \label{fig:figure2}
\end{figure*}

\subsubsection{Metrics}

To assess the performance of the compared methods, we employed three widely used image quality metrics: Peak Signal-to-Noise Ratio (PSNR), Structural Similarity Index Measure (SSIM) \cite{wangImageQualityAssessment2004}, and Learned Perceptual Image Patch Similarity (LPIPS) \cite{zhangUnreasonableEffectivenessDeep2018}. PSNR measures pixel-wise differences between the reconstructed and original images, favoring reconstructions with lower pixel-level errors, even if they appear overly smooth or perceptually less realistic. In contrast, SSIM and LPIPS focus on perceptually relevant differences, with LPIPS leveraging the deep features of a neural network to evaluate perceptual similarity. However, LPIPS and SSIM alone may struggle to account for hallucinated details that can appear in reconstructions.

\subsubsection{Hyper-parameter tuning}

Our algorithm includes several tunable hyper-parameters that significantly influence the quality of the results. Among these, our experiments identified the noise level $\sigma$ of the GS denoiser and the regularization parameter $\lambda$ as the most critical. Although its impact is less significant, the gradient step size $\tau$ was also optimized. To ensure a fair comparison, we adopted the methodology of Hurault \textit{et al.} \cite{huraultConvergentBregmanPlugandplay2024a}, tuning all three parameters using the CBSD68 dataset and all blur kernels, for each noise level. During the grid search over the hyperparameter space, the algorithm was run for 800 iterations to ensure convergence across all noise levels and images. 

We restricted the range of hyper-parameters to stay within the convergence bounds. We initially implemented a back-tracking strategy for the descent step-size $\tau$ but found that it never activates in our experiments when $\tau \leq 1$. Additionally, the GS denoiser is shown in \cite{huraultGradientStepDenoiser2021} (Appendix B) to have a local Lipschitz constant slightly larger than 1 for most images. This leads us to simply constrain $\tau, \lambda \in (0, 1)$ to satisfy the condition $\tau \lambda < \frac{1}{L}$. Regarding the choice of the noise parameter $\sigma$ for the regularization $g_\sigma$, since the denoiser was trained on noise levels $\sigma \in [0, 50 / 255]$, we maintained the same bounds during the grid search. 
Finally, we also conducted a grid search to determine the optimal regularization hyperparameter for the RL-TV method as well as the DPIR method's $\sigma$ parameter at each noise level, under the same conditions as previously outlined. A summary of the optimal hyper-parameters for each noise level is provided in Table \ref{annex:tab:deblur-hyperparam} of Appendix \ref{annex:hyperparams}.


\subsubsection{Qualitative and quantitative results}

Qualitatively, our method achieves substantially better detail recovery than all competing approaches across both low- and high-noise regimes, as shown in Figure \ref{fig:figure1}.
These observations are consistent across the validation and test datasets, with quantitative results provided in Table \ref{tab:deblur-kodak-quantitative}. Additionnal quantitative results on the CBSD69 dataset can be found in Appendix \ref{annex:additional-exp} Table \ref{annex:tab:deblur-cbsd68-quantitative}.
Our PnP-MM algorithm outperforms all baselines in every experiment, establishing a new state-of-the-art for deblurring under Poisson noise. On average, it achieves the highest PSNR and the best perceptual scores, in line with the visual quality observed on individual reconstructions. We also note that our method is particularly stable to high noise levels.
Among the competing methods, DPIR \cite{zhangPlugPlayImageRestoration2022} still performs reasonably well under mild noise and often ranks second.
The unrolled PhDNet architecture \cite{sanghviPhotonLimitedNonBlind2022} is both fast and inexpensive to run in terms of memory. It has no hyper-parameters, avoiding the need for time-consuming tuning. Finally, it yields fair reconstructions across all noise levels, although it remains consistently below the performance of our method.
DPS \cite{chungDiffusionPosteriorSampling2022} adapted to Poisson noise, is still unstable in low SNR regimes and outputs many hallucinated details which results in poor performance in terms of PSNR. We found the ProxDiffPIR method \cite{melidonis2025scorebased} gives far superior reconstructions than DPS at all noise levels, at the cost of increased computations: it performs an inner optimization loop at each sampling step, and averages the output over 3 samples. In average, our method outperforms the convergent Bregman-PnP \cite{huraultConvergentBregmanPlugandplay2024a} by more than 2dB on high Poisson noise ($\zeta=5)$.

\subsection{Positron Emission Tomography}
\label{subsec:exp:pet}

\begin{figure*}[h]
    \includegraphics[width=\linewidth]{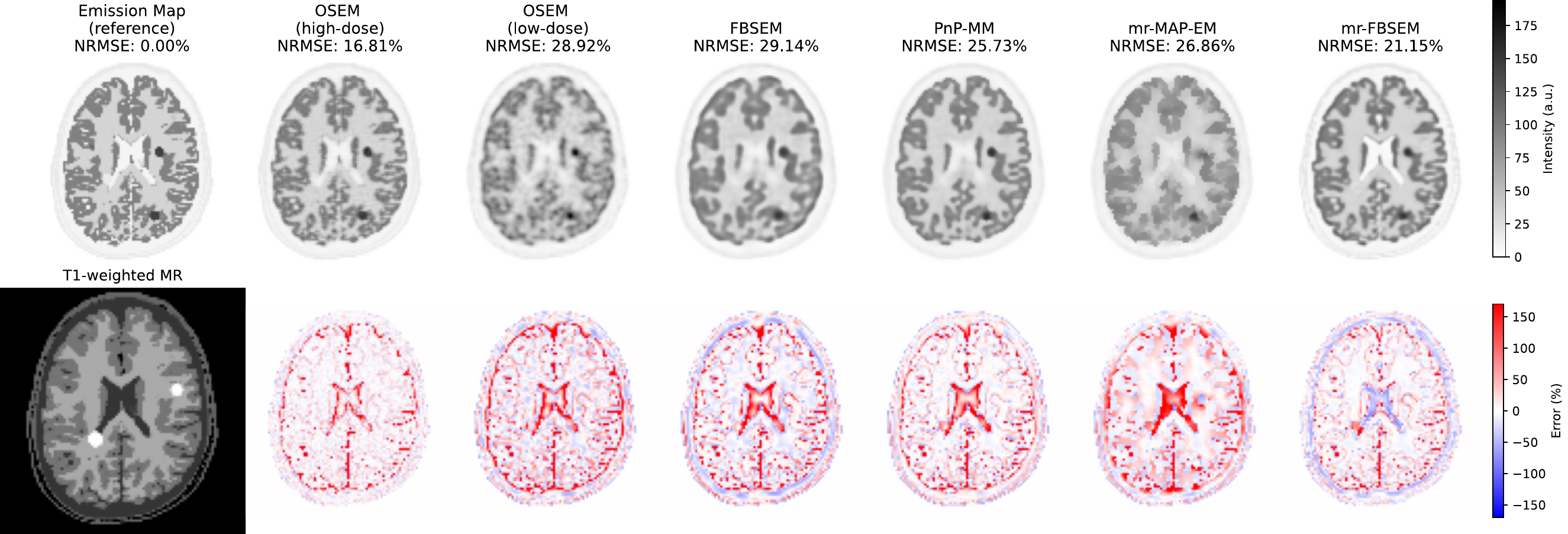} 
    \caption{Top row: PET reconstructions of a representative mid-axial slice from the BrainWeb test dataset, including two hot lesions. Bottom row: MR slice and the corresponding relative error maps for each reconstruction algorithm with respect to the emission map.}
    \label{fig:pet-figure1}
\end{figure*}

\subsubsection{Dataset and experimental setting}

We used the BrainWeb phantom dataset. The first 10 phantoms were used for training and the remaining 10 for testing. For each phantom, we inserted 15 circular cold and hot lesions with radius ranging from 2 to 8 mm at random spatial locations. Hot lesions have 1.5 times the intensity of surrounding gray matter while cold lesions have half the intensity of surrounding white matter. The training set is composed of 100 axial slices of size 172$\times$172 pixels taken from 10 volumes, and is augmented by applying 5 random in-plane rotations sampled uniformly within $\pm 15^\circ$.
This yields 500 slices for training. The same procedure was applied to generate 500 slices for testing. For every slice, we simulated a 100M-count reference sinogram and a low-dose sinogram with counts varying from 1M to 10M.
We took inspiration from the simulation protocol of Mehranian et \textit{al.} \cite{mehranianModelBasedDeepLearning2021}.
All projections were simulated following the geometry of the Siemens Biograph mMR scanner, including point-spread-function (PSF) modeling, attenuation, and normalization.
The reference sinogram uses a 2.5 mm full-width-at-half-maximum (FWHM) PSF and is reconstructed using 15 iterations of the Ordered Subset Expectation Maximization (OSEM) algorithm \cite{hudsonAcceleratedImageReconstruction1994} with 14 subsets. The low-dose data uses a 4 mm FWHM PSF and is reconstructed using 10 OSEM iterations with 14 subsets. 
Random lesions are also added to the T1-weighted MR image associated with each phantom. During reconstruction, the low-dose PSF modeling, attenuation correction and normalization are taken into account.

\begin{figure}[h]
    \centering
    \includegraphics[width=\linewidth]{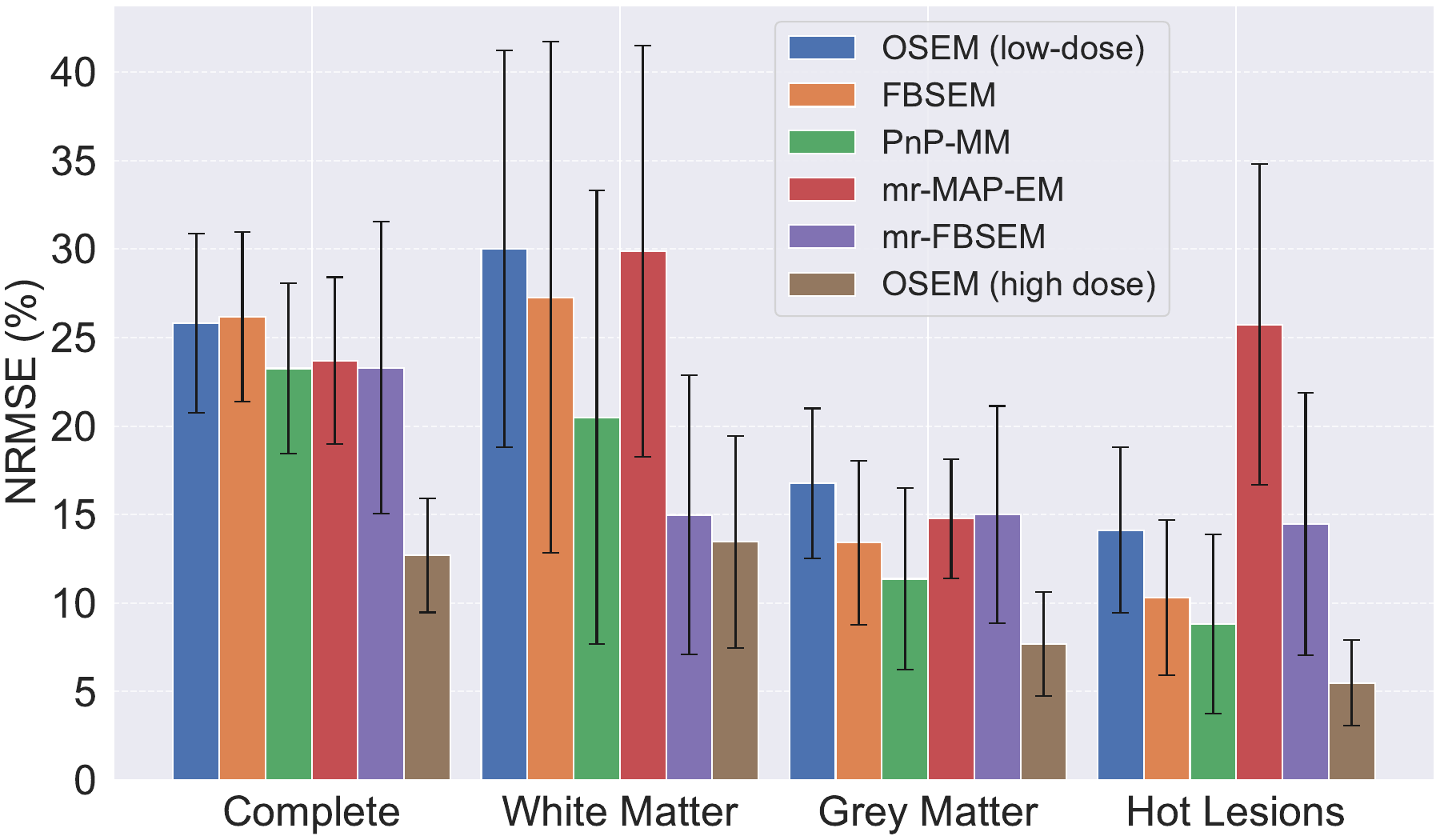}
    \caption{NRMSE values for each method and for different region of interest  on the Brainweb test dataset}
    \label{fig:pet-figure2}
\end{figure}

\subsubsection{Denoiser}
For this experiment, we retrain a Gaussian GS-denoiser on the PET training dataset.
In contrast to the procedure of \cite{zhangPlugPlayImageRestoration2022}, we do not normalize the images to the $\left[0, 1\right]$ range. While such normalization is appropriate for natural images, it is not well suited to PET data, where voxel intensities correspond to Standardized Uptake Values (SUV) and depend on the underlying count level or, in clinical settings, on uptake time. Normalizing each high-count reconstruction independently would significantly distort the dynamic range and introduce large, non-physical variations in contrast across samples.
Instead, we trained the denoiser directly on raw SUV values. The noise amplitudes were chosen to approximately match the signal-to-noise ratios used in natural-image denoising: while natural images typically span $[0, 255]$ with Gaussian noise levels $\sigma \in [0, 50]$, SUV values in our dataset exhibit a slightly smaller dynamic range. We therefore adopt Gaussian noise standard deviations in the range $\sigma \in [0, 40]$. We trained the neural network on patches of size 128$\times$128 for 500 epochs with an $L_2$ loss using Adam. The training took a total of 4 hours on a Quadro RTX 8000 GPU. The resulting denoiser was evaluated on the test set across multiple Gaussian noise levels, as reported in Appendix \ref{annex:additional-exp}.

\begin{figure}[h]
    \centering
    \includegraphics[width=\linewidth]{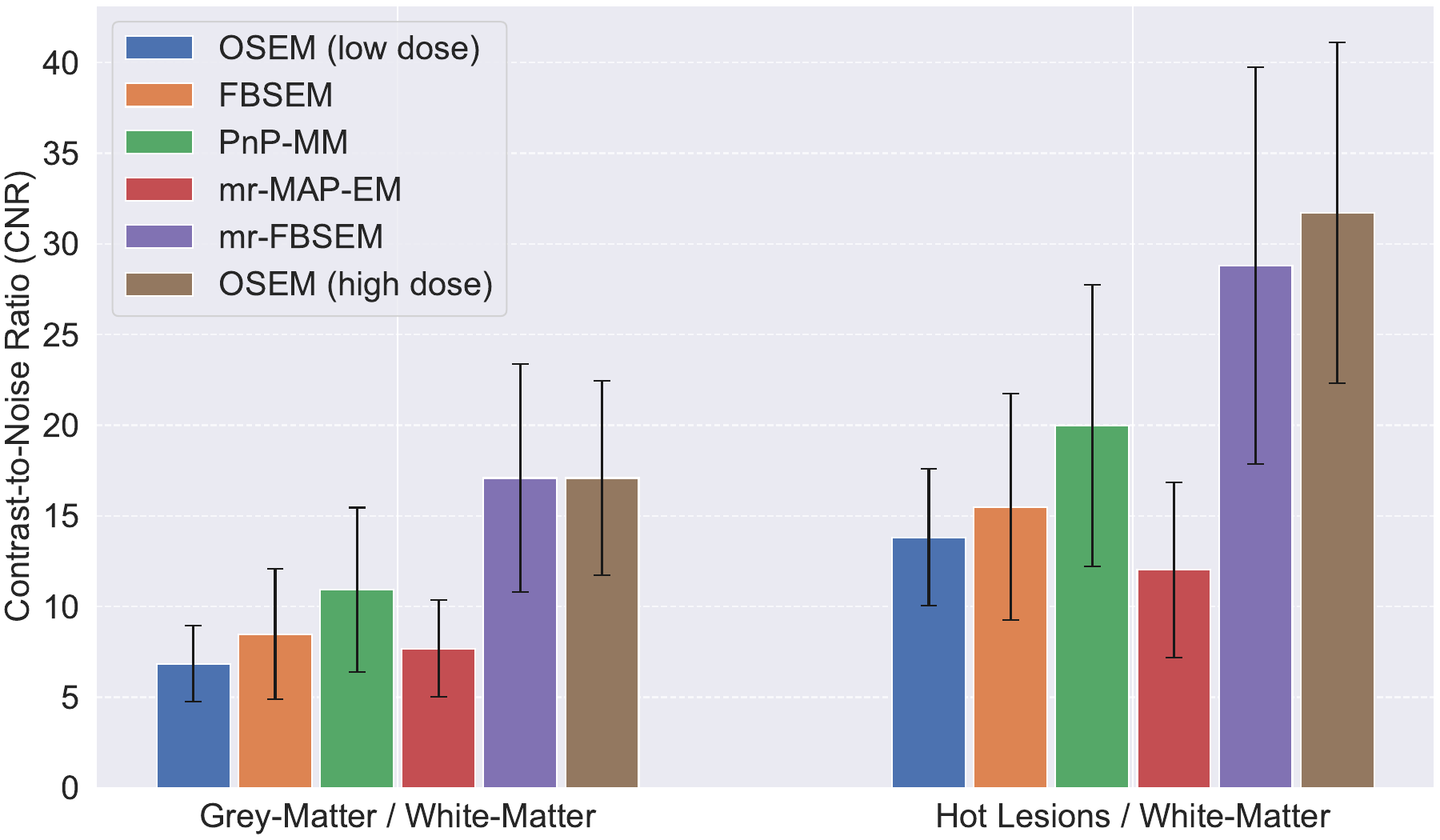}
    \caption{Contrast-to-Noise values for each method on the Brainweb test dataset}
    \label{fig:pet-figure3}
\end{figure}

\begin{figure*}[t]
    \includegraphics[width=\linewidth]{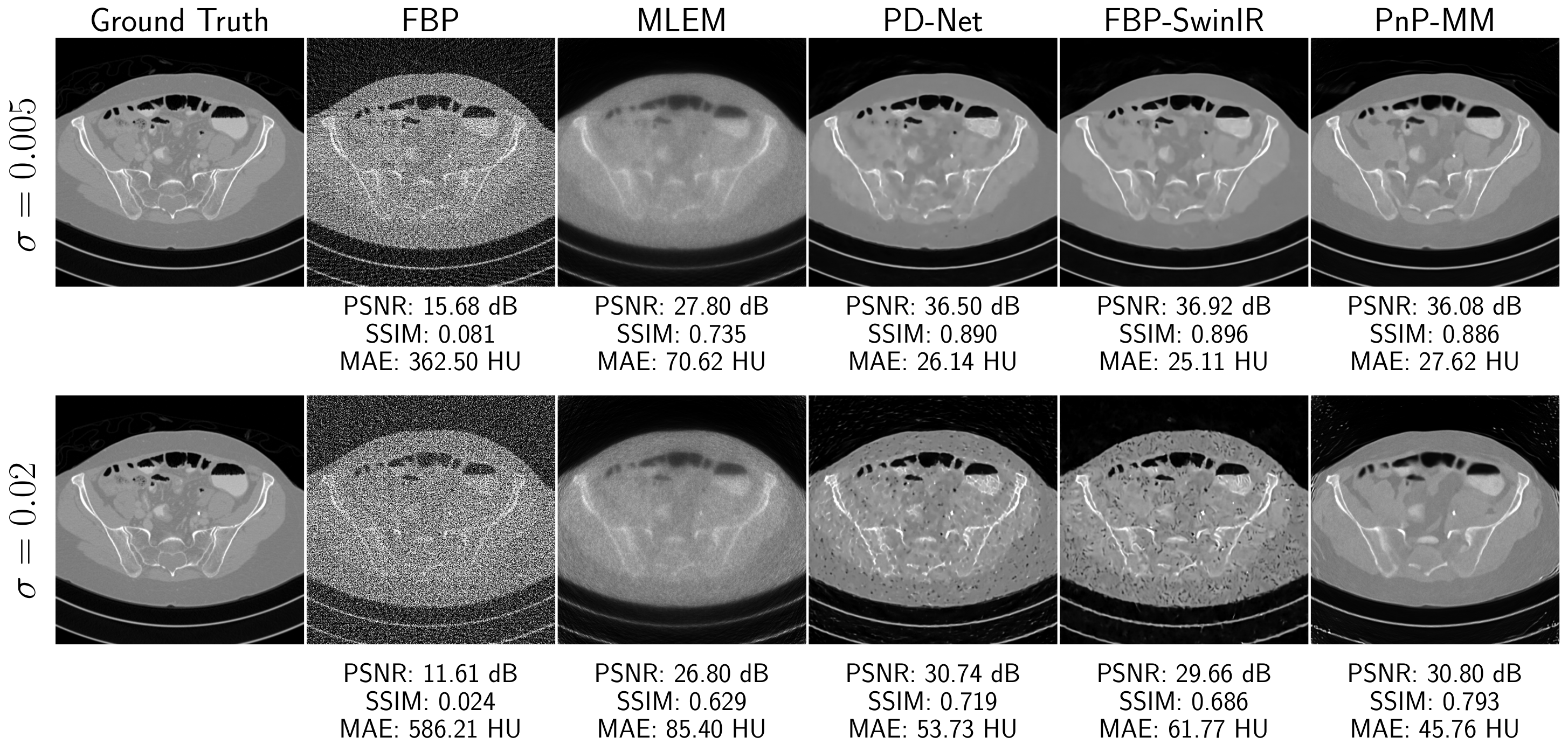} 
    \caption{Reconstruction of a test slice with the different methods. Top row: 128 views with the lowest Gaussian noise level, which was used to train the fully-supervised methods PD-Net and FBP-SwinIR. Bottom row: 128 views with the highest Gaussian noise level, which was never seen during the training of the fully supervised methods.}
    \label{fig:ct-figure1}
\end{figure*}

\subsubsection{Methods}
We compared against common methods in PET reconstruction as well as recent unrolled approaches. First, we considered  the standard OSEM algorithm, an accelerated variant of MLEM that processes ordered subsets of the projection data instead of the full dataset \cite{hudsonAcceleratedImageReconstruction1994}.
We also evaluated a regularized reconstruction using a Bowsher prior computed from the corresponding T1-weighted MR image \cite{bowsherUtilizingMRIInformation2004} which we denote as mr-MAP-EM.
This method is close to ours in the sense that it is also based on the EM majorant. It differs by the use of $x^{(n)}$ in the computation of $x_{\text{EM}}$, as in DePierro's algorithm \cite{pierroModifiedExpectationMaximization1995}. Finally, we compared against a state-of-the-art deep-learning approach that unrolls these iterations \cite{mehranianModelBasedDeepLearning2021}.
This framework has two variants: mr-FBSEM, which incorporates T1-MR information, and FBSEM, which uses PET data alone. For mr-FBSEM we used the publicly released pretrained weights of the 2D implementation \cite{mehranianModelBasedDeepLearning2021}. As the weights for the FBSEM unrolled network were not available, we retrained it on the training dataset.

\subsubsection{Metrics}
We first computed the Root Normalized Mean Squared Error (RNMSE) between reconstructions and the emission map over the whole image, as well as only on hot lesions, grey and white matter regions. We also computed the Contrast to Noise Ratio (CNR) between grey and white matter, by taking the average activity in the grey matter area minus the average activity in the white matter area divided by the standard deviation in the white matter area. The CNR was also evaluated for hot lesions over white matter.

\subsubsection{Hyper-parameter tuning}

For this experiment, we tuned the hyperparameters of our method using five validation slices and taking the set of hyper-parameters yielding the best overall RMNSE.  
We observed that the algorithm converges slowly when the step size was constrained to values below 1.
To mitigate this, we retained a unit step size for the GS-denoiser update but increased the step size $\tau$ of the data-fidelity update to 100.
This yields a substantial acceleration in convergence without introducing noticeable numerical instabilities. All hyper-parameters are reported in Table \ref{annex:tab:pet-hyperparam} of Appendix \ref{annex:hyperparams}. 

\subsubsection{Qualitative and quantitative results}

Qualitative results are presented in Figure \ref{fig:pet-figure1}.
Our method achieves improved recovery of lesion details compared with FBSEM and outperforms the Bowsher-prior reconstruction, as reported in Figure \ref{fig:pet-figure2}, despite the latter leveraging MR information. As expected, mr-FBSEM, which exploits MR images during both training and inference, produces sharper anatomical structures and higher overall contrast. However, SUV values are overestimated in some slices, leading to very high NRMSE values, as evidenced by the large standard deviations reported for full images in Figure \ref{fig:pet-figure2}. 
Quantitatively, our method yields the lowest NRMSE on both full slices and region-of-interest evaluations, including grey matter and hot lesions. While the CNR between grey and white matter remains lower than that of MR-guided FBSEM, a substantial improvement is observed over the PET-only variant (Figure \ref{fig:pet-figure3}).

\subsection{Computed Tomography}
\label{subsec:exp:ct}

\subsubsection{Datasets \& Experimental setting}

\begin{table*}[h]
\caption{Comparison on AAP Mayo test dataset}
\label{tab:ct-mayo}
\resizebox{\textwidth}{!}{%
\begin{tabular}{lc|ccc|ccc|ccc|}
\multicolumn{1}{c}{}           &                                                       & \multicolumn{3}{c|}{192 angles}                                                   & \multicolumn{3}{c|}{128 angles}                                                   & \multicolumn{3}{c|}{64 angles}                                                   \\ \hline
\multicolumn{1}{|c|}{$\sigma$} & Methods                                               & PSNR $\uparrow$           & SSIM $\uparrow$            & MAE (HU) $\downarrow$    & PSNR $\uparrow$           & SSIM $\uparrow$            & MAE (HU) $\downarrow$    & PSNR $\uparrow$           & SSIM $\uparrow$            & MAE (HU) $\downarrow$   \\ \hline
\multicolumn{1}{|l|}{0.005}    & FBP                                                   & 17.39 $\pm$ 0.72          & 0.119 $\pm$ 0.023          & 290 $\pm$ 31.2           & 16.60 $\pm$ 0.71          & 0.103 $\pm$ 0.021          & 318 $\pm$ 34.0           & 15.07 $\pm$ 0.66          & 0.077 $\pm$ 0.017          & 378 $\pm$ 38.9          \\
\multicolumn{1}{|l|}{}         & MLEM \cite{langeEMReconstructionAlgorithms1984a}      & 28.72 $\pm$ 1.06          & 0.772 $\pm$ 0.028          & 62.7 $\pm$ 10.2          & 28.67 $\pm$ 1.06          & 0.765 $\pm$ 0.028          & 63.1 $\pm$ 10.3          & 28.27 $\pm$ 1.04          & 0.735 $\pm$ 0.032          & 66.7 $\pm$ 10.7         \\
\multicolumn{1}{|l|}{}         & PDNet \cite{adlerLearnedPrimalDualReconstruction2018} & {\ul 37.28 $\pm$ 0.97}    & {\ul 0.897 $\pm$ 0.018}    & {\ul 24.7 $\pm$ 3.2}     & {\ul 37.05 $\pm$ 0.99}    & {\ul 0.893 $\pm$ 0.018}    & {\ul 25.2 $\pm$ 3.2}     & {\ul 34.58 $\pm$ 1.08}    & {\ul 0.853 $\pm$ 0.022}    & {\ul 33.4 $\pm$ 4.3}    \\
\multicolumn{1}{|l|}{}         & FBP-SwinIR \cite{Liu_2021_ICCV}                       & \textbf{37.66 $\pm$ 1.01} & \textbf{0.903 $\pm$ 0.018} & \textbf{23.50 $\pm$ 3.1} & \textbf{37.31 $\pm$ 1.01} & \textbf{0.898 $\pm$ 0.018} & \textbf{24.4 $\pm$ 3.2}  & 31.82 $\pm$ 1.53          & 0.796 $\pm$ 0.040          & 45.3 $\pm$ 8.0          \\
\multicolumn{1}{|l|}{}         & PnP-MM                                                & 36.40 $\pm$ 1.09          & 0.892 $\pm$ 0.017          & 26.53 $\pm$ 3.3          & 36.44 $\pm$ 0.94          & 0.888 $\pm$ 0.016          & 26.9 $\pm$ 3.1           & \textbf{35.64 $\pm$ 1.19} & \textbf{0.873 $\pm$ 0.020} & \textbf{28.6 $\pm$ 4.0} \\ \hline
\multicolumn{1}{|l|}{0.01}     & FBP                                                   & 15.58 $\pm$ 0.52          & 0.077 $\pm$ 0.014          & 358 $\pm$ 30.9           & 15.07 $\pm$ 0.53          & 0.069 $\pm$ 0.013          & 380 $\pm$ 33.4           & 13.97 $\pm$ 0.52          & 0.055 $\pm$ 0.011          & 431 $\pm$ 37.9          \\
\multicolumn{1}{|l|}{}         & MLEM \cite{langeEMReconstructionAlgorithms1984a}      & 28.39 $\pm$ 0.99          & 0.744 $\pm$ 0.027          & 67.2 $\pm$ 10.1          & 28.37 $\pm$ 1.0           & 0.739 $\pm$ 0.027          & 67.0 $\pm$ 10.1          & 28.14 $\pm$ 1.01          & 0.723 $\pm$ 0.028          & 68.6 $\pm$ 10.5         \\
\multicolumn{1}{|l|}{}         & PDNet \cite{adlerLearnedPrimalDualReconstruction2018} & {\ul 36.12 $\pm$ .87}     & {\ul 0.876 $\pm$ 0.018}    & {\ul 28.7 $\pm$ 3.1}     & {\ul 35.97 $\pm$ 0.9}     & {\ul 0.873 $\pm$ 0.018}    & {\ul 29.0 $\pm$ 3.1}     & {\ul 33.64 $\pm$ 1.00}    & {\ul 0.824 $\pm$ 0.020}    & {\ul 37.9 $\pm$ 4.2}    \\
\multicolumn{1}{|l|}{}         & FBP-SwinIR \cite{Liu_2021_ICCV}                       & \textbf{36.85 $\pm$ 0.99} & \textbf{0.892 $\pm$ 0.019} & \textbf{26.1 $\pm$ 3.3}  & \textbf{36.41 $\pm$ 0.96} & \textbf{0.884 $\pm$ 0.020} & \textbf{27.4 $\pm$ 3.36} & 30.74 $\pm$ 1.14          & 0.761 $\pm$ 0.044          & 51.9 $\pm$ 8.6          \\
\multicolumn{1}{|l|}{}         & PnP-MM                                                & 35.06 $\pm$ 0.85          & 0.873 $\pm$ 0.017          & 31.0 $\pm$ 3.2           & 35.17 $\pm$ 0.93          & 0.870 $\pm$ 0.018          & 30.4 $\pm$ 3.4           & \textbf{34.92 $\pm$ 1.10} & \textbf{0.862 $\pm$ 0.020} & \textbf{30.6 $\pm$ 3.9} \\ \hline
\multicolumn{1}{|l|}{0.02}     & FBP                                                   & 12.22 $\pm$ 0.26          & 0.031 $\pm$ 0.005          & 528 $\pm$ 31.9           & 12.01 $\pm$ 0.27          & 0.029 $\pm$ 0.005          & 542 $\pm$ 33.4           & 11.49 $\pm$ 0.28          & 0.025 $\pm$ 0.005          & 575 $\pm$ 36.3          \\
\multicolumn{1}{|l|}{}         & MLEM \cite{langeEMReconstructionAlgorithms1984a}      & 27.38 $\pm$ 0.81          & 0.664 $\pm$ 0.025          & 79.2 $\pm$ 9.6           & 27.46 $\pm$ 0.83          & 0.663 $\pm$ 0.025          & 77.6 $\pm$ 9.7           & 27.42 $\pm$ 0.87          & 0.660 $\pm$ 0.026          & 76.7 $\pm$ 10.1         \\
\multicolumn{1}{|l|}{}         & PDNet \cite{adlerLearnedPrimalDualReconstruction2018} & 30.17 $\pm$ 0.36          & 0.704 $\pm$ 0.018          & 56.25 $\pm$ 2.9          & 30.62 $\pm$ 0.49          & 0.720 $\pm$ 0.014          & 53.3 $\pm$ 2.9           & {\ul 30.14 $\pm$ 0.78}    & {\ul 0.705 $\pm$ 0.016}    & {\ul 56.9 $\pm$ 4.62}   \\
\multicolumn{1}{|l|}{}         & FBP-SwinIR \cite{Liu_2021_ICCV}                       & {\ul 31.62 $\pm$ 0.67}    & {\ul 0.768 $\pm$ 0.040}    & {\ul 47.64 $\pm$ 5.2}    & {\ul 30.66 $\pm$ 0.71}    & {\ul 0.738 $\pm$ 0.047}    & {\ul 53.1 $\pm$ 6.3}     & 26.67 $\pm$ 0.89          & 0.606 $\pm$ 0.061          & 84.9 $\pm$ 10.9         \\
\multicolumn{1}{|l|}{}         & PnP-MM                                                & \textbf{31.99 $\pm$ 0.64} & \textbf{0.842 $\pm$ 0.021} & \textbf{40.14 $\pm$ 3.9} & \textbf{32.17 $\pm$ 0.65} & \textbf{0.831 $\pm$ 0.020} & \textbf{39.0 $\pm$ 3.8}  & \textbf{31.47 $\pm$ 0.72} & \textbf{0.802 $\pm$ 0.020} & \textbf{41.3 $\pm$ 4.2} \\ \hline
\end{tabular}
}

\end{table*}

We conducted all our experiments on the training split of the 2016 NIH-AAPM-Mayo Clinic Low Dose CT Grand Challenge dataset \cite{mccolloughLowdoseCTDetection2017}. This enables us to use high-quality full-dose CT scans to compute metrics. We selected 2886 full dose slices of size 512$\times$512 pixels, corresponding to the first 5 patients, all acquired with a 1 mm slice thickness and a B30 reconstruction kernel. The remaining five patients were reserved for testing, yielding a total of 3,050 test slices. For each slice, we generated six sinograms using 64, 128, and 192 parallel projection angles with a normalized forward operator.
For every acquisition setting, we simulated photon noise with parameter $\zeta = 1 / 3000$ and additive electronic noise modeled as Gaussian noise with standard deviation equal to $\sigma=0.5\%, 1\%, 2\%$ of the measured signals.
These experiments were designed to assess the robustness of our method to variations in the forward operator and to the presence of additional Gaussian noise, which would not be negligible in a low-dose acquisition.
We note, however, that real CT transmission physics modifies the forward model to $\exp(-Ax)$. Adapting the majorizing function to this setting is nontrivial and lies beyond the scope of the present work.

\subsubsection{Denoiser}

For CT it is common to clip the images between a fixed Housefield Units (HU) range. For these experiments, we clip all slices to the interval $[-1000, 2000]$. This enables us to rescale all images to $[0, 1]$ without running into normalization issues, and to use the same Gaussian noise levels as in the natural image setting. We retrained a DRUNet and a GS-DRUNet on the training split, for 500 epochs, with Gaussian noise in range $[0, 50 / 255]$. We noticed that training on patches of size 256$\times$256 instead of 128$\times$128 significantly improves the denoising performances. The training takes around 48 hours on a NVIDIA RTX A6000 GPU.

\subsubsection{Methods}

We compared our method against both classic baselines and state-of-the-art approaches for CT reconstruction.
First, we included the standard filtered backprojection (FBP) as well as plain MLEM applied to the shifted-Poisson likelihood.
We then compared with state-of-the-art fully supervised deep learning methods: an unrolled primal–dual architecture with convolutional modules called PD-Net \cite{adlerLearnedPrimalDualReconstruction2018}, and a Transformer-based model trained to recover full-dose references from FBP reconstructions called FBP-SwinIR \cite{Liu_2021_ICCV}.

\subsubsection{Metrics}

For quantitative evaluation we used the PSNR and SSIM metrics, as well as the Mean Absolute Error (MAE). Each reconstruction was clipped between $[-1000, 2000]$ HU before computing the metrics.

\subsubsection{Hyper-parameter tuning}

We early stopped the MLEM algorithm to 20 iterations to avoid artifacts. For our algorithm, we set the step-size $\tau=1$ and we optimized the noise level $\sigma$ as well as the regularization parameter $\lambda$ on 4 validation slices and tuned them against MAE. These different hyper-parameters are reported in Table \ref{annex:tab:ct-hyperparam} of Appendix \ref{annex:hyperparams}. In all CT experiments, we initialized the algorithm with a uniform input of 1, and we performed 600 iterations of the algorithm, as we do not use any ordered-subset strategy.
\begin{figure*}[t]
    \includegraphics[width=\linewidth]{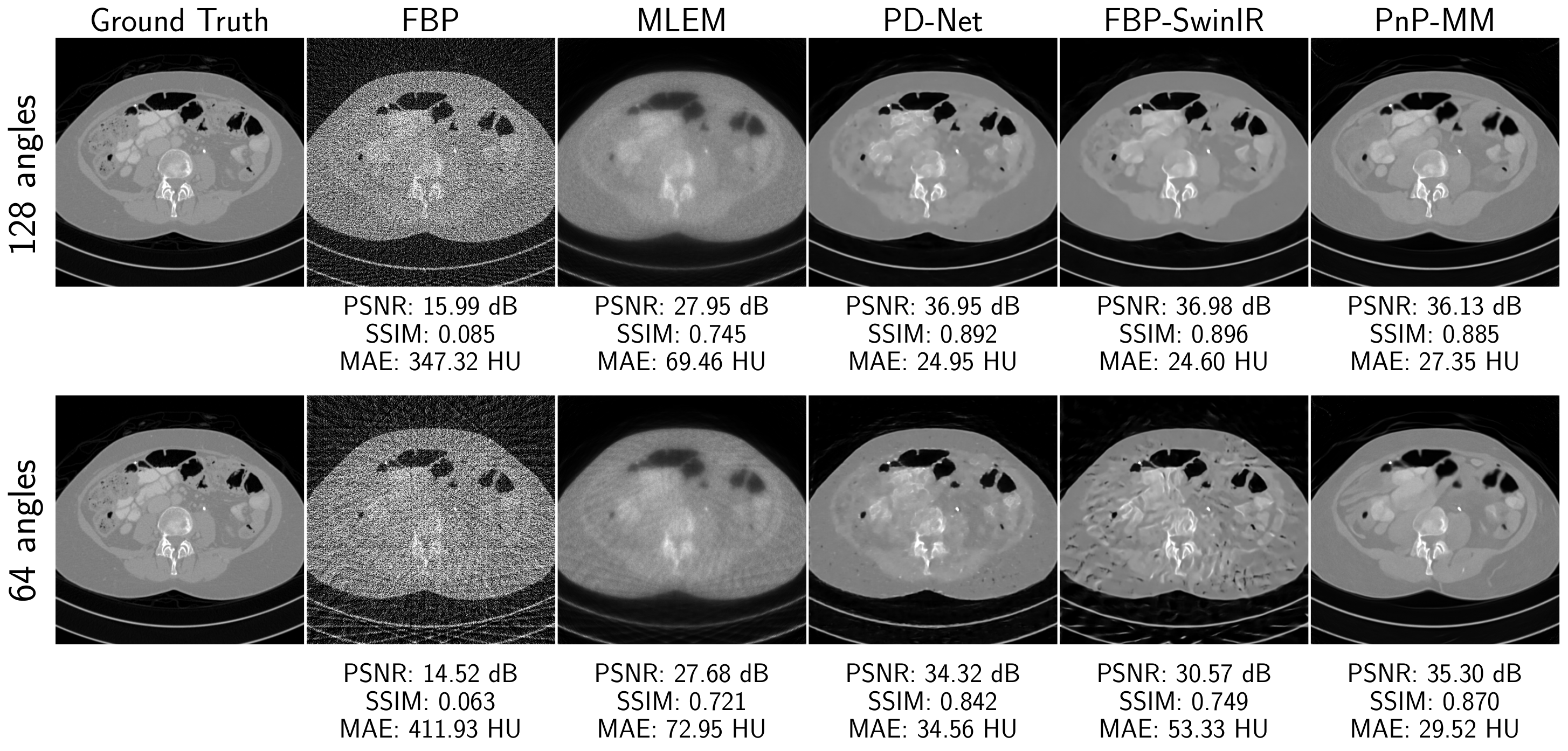} 
    \caption{Reconstruction of a test slice with the different methods. Top row: 128 views with the lowest Gaussian noise level, which was used to train the fully-supervised methods. Bottom row: 64 views with the lowest Gaussian noise level, which was never seen during the training of the fully supervised methods.}
    \label{fig:ct-figure2}
\end{figure*}

\subsubsection{Qualitative and quantitative results}

Qualitative results are shown in Figures~\ref{fig:ct-figure1} and ~\ref{fig:ct-figure2}, while quantitative results are reported in Table \ref{tab:ct-mayo}. The fully supervised FBP-SwinIR method achieves the best performance when evaluated under acquisition conditions close to its training distribution, notably for low noise levels ($\sigma=0.005$) and moderately sparse angular sampling (128 or 192 views). In these regimes, FBP-SwinIR consistently yields the highest PSNR and SSIM and the lowest MAE. Similar trends are observed for $\sigma=0.01$ and 192 angles, indicating that the method remains robust to moderate increases in noise. Adding more acquisition angles also seem to have little impact on the performances.

However, degradations are observed when reducing the number of projection angles to 64 as can be seen in Figure \ref{fig:ct-figure2}. In this sparse-view regime, the performance of FBP-SwinIR drops sharply across all noise levels. PD-Net exhibits a similar but less severe sensitivity to the reduction in the number of projections.

In contrast, while it does not achieve the best performances in the setting matching the training of FBP-SwinIR and PD-Net, the proposed PnP-MM method demonstrates strong robustness to variations in both the acquisition geometry and the noise level.

\section{Discussions}
\label{sec:discussions}

Although the proposed method reliably converges to satisfactory stationary points of the objective function, the convergence rate is relatively slow and typically requires several hundred iterations. This limitation could be alleviated by integrating established acceleration techniques such as Nesterov’s momentum  \cite{nesterovMethodSolvingConvex1983}, or by incorporating ordered-subsets schemes similar to those explored in our PET experiments.

For the experiments on natural images, we rely on a Gaussian denoiser which was trained on a dataset comprising several thousand samples. This rich prior enables the method to achieve state-of-the-art performance in Poisson deconvolution. In contrast, the PET and CT experiments are conducted on substantially smaller datasets, each containing only five patients. This reduces the expressive power of the learned prior, which leads to a choice of substantially smaller regularization weights $\lambda$, as reported in Appendix \ref{annex:hyperparams}. While these results illustrate that PnP methods remain competitive even in low-data regimes, practical use in medical imaging would clearly benefit from denoisers trained on larger and more diverse databases.

Our current implementation uses fixed hyperparameters throughout the iterations, which limits the reliance on heuristics but may not fully exploit the problem structure. Adaptive strategies, analogous to those employed in diffusion or accelerated PnP methods, could potentially improve performance. While our algorithm is reasonably robust to the choice of hyperparameters, tuning them remains time-consuming, and developing principled, automated selection mechanisms, for instance inspired by \cite{cascaranoConstrainedRegularizationDenoising2024a}, can make the method easier to apply.

Finally, all experiments in this work were performed on 2D slices. In medical imaging, fully 3D acquisitions are the standard, and extending our method to 3D is a necessary step. Beyond aligning with clinical practice, a 3D formulation would naturally improve reconstruction quality in settings where the measurement geometry is inherently volumetric, such as cone-beam CT or PET. A key advantage of PnP approaches is that this transition to 3D only modestly grows the training cost, especially compared with unrolled approaches.

\section{Conclusion}
\label{sec:conclusion}

In this paper, we introduced a novel algorithm for solving linear inverse problems corrupted by Poisson noise, and extended it to the Poisson–Gaussian setting through a shifted-Poisson approximation. The proposed method draws upon classical Majorization–Minimization principles and recent advances in provably convergent PnP schemes. We showed that convergent PnP algorithms can be constructed beyond the Gaussian noise model by exploiting suitable tangent majorants of the data-fidelity term, without modifying the prior. In particular, we specialized this approach to the Kullback–Leibler divergence by relying on its canonical EM majorant, which enables the integration of Gaussian denoisers while preserving convergence to a stationary point of an explicit objective, even under Poisson noise.
Our experiments demonstrate that, unlike fully supervised approaches, the proposed method exhibits strong robustness to variations of the forward operator and remains stable under severe Poisson and Poisson–Gaussian noise. This work opens several promising directions. Future research will focus on extending the approach to 3D data which are common in medical imaging.

\bibliography{bib/biblio}{}
\bibliographystyle{bib/ieeetr_custom}

\newpage

\section{Appendix}

\subsection{Derivation of the Richardson-Lucy / MLEM updates}
\label{sec:proof:lemma-mu}

Let us first revisit the key lemma that shows the MM framework's ability to guarantee a decrease in the objective function.

\begin{lemma}
\label{lem:mm}
If $F$ is a tangent majorant of $f$, then the sequence $(f(x^{(n)}))$ is non-increasing under the updates:
\begin{equation}
    x^{(n+1)} = \argmin_x F(x, x^{(n)}).
    \label{eq:annex2}
\end{equation}
\end{lemma}

\begin{proof} From definition \ref{def:surrogate} and (\ref{eq:annex2}) we obtain:
    \begin{equation*}
        f(x^{(n+1)}) \leq F(x^{(n+1)}, x^{(n)}) \leq F(x^{(n)}, x^{(n)}) = f(x^{(n)}).
    \end{equation*}
\end{proof}

If $f$ is the Poisson NLL as defined in (\ref{eq:poisson-likelihood}), then replacing $x$ by $\tilde{x}$ in the expression of the tangent majorant $F_{\text{EM}}$ (\ref{eq:surrogate}) directly gives $F_{\text{EM}}(\tilde{x},\tilde{x}) = f(\tilde{x})$. 
The proof of the majoration property $\forall \tilde{x}, \in \mathbb{R}^{d}_{+}$, $f(x) \leq F_{\text{EM}}(x, \tilde{x})$ relies on a Jensen convexity inequality. Indeed one can see that since $\sum_j \widetilde\alpha_{ij} = 1$, 
\begin{equation*}
    - \sum_j y_i \tilde{\alpha}_{ij} \log \left( \frac{A_{ij} x_j}{\tilde{\alpha}_{ij} } \right) \geq - y_i \log \left(\sum_j A_{ij} x_j \right),
\end{equation*}
and then:
\begin{equation*}
    F_{\text{EM}}(x, \tilde{x}) \geq \sum_i \left((Ax)_i -  y_i \log \left( (Ax)_i \right)\right) = f(x).
\end{equation*}
Now the minimization of $F$ with respect to its first argument is also straightforward. The partial derivatives of $F$ are:
\begin{equation*}
    \frac{\partial F_{\text{EM}}(x, \Tilde{x})}{\partial x_j} = \sum_i \left[ A_{ij} - \frac{\Tilde{\alpha}_{ij} y_i}{x_j} \right]
    =s_j - \frac{\Tilde{x}_j}{x_j}\sum_i \frac{A_{ij}y_i}{(A\Tilde{x})_i}.
\end{equation*}
The Karush-Kuhn-Tucker (KKT) condition $x_j\frac{\partial F_{\text{EM}}(x, \Tilde{x})}{\partial x_j} = 0$ then gives:
\begin{equation*}
    x_j = \frac{\Tilde{x}_j}{s_j} \sum_i \frac{y_i A_{ij}}{(A\Tilde{x})_i},
\end{equation*}
meaning we have the expected result: 
\begin{equation*}
    x^{*} = \frac{\tilde{x}}{s} \cdot A^T \frac{y}{A \tilde{x}}.
\end{equation*}

\subsection{Proof of Theorem \ref{th:convergence-pnp}}
\label{sec:proof:pnp-convergence}

    1) Our analysis builds upon well-established convergence results for the forward-backward splitting method, as described in \cite{ beckFirstOrderMethodsOptimization2017}, along with the basic properties of tangent majorants. We will use the MM framework to construct a tangent majorant of $h(x) = f(x) + g(x)$. Adding the regularization parameter $\lambda$ is straightforward by replacing $g$ by $\lambda g$ and the Lipschitz constant $L$ by $\lambda L$, and is not detailed here for ease of notation.
    We define $H$ the tangent majorant of $h$: 
    \begin{equation*}
        H (x, \tilde{x}) = F(x, \tilde{x}) +  g(\tilde{x}) + \langle x - \tilde{x}, \nabla g (\tilde{x}) \rangle + \frac{1}{2 \tau} \| x - \tilde{x}\|^2,
    \end{equation*}
    where $\tau > 0$.
    Given that $F(x, x) = f(x)$, showing that $H (x, x) = h(x)$ is trivial. We will now show that for all $x, \tilde{x} \in \mathbb{R}^d, h(x) \leq H(x, \tilde{x})$. 
    Since $g$ is $L$-smooth, the descent lemma \cite{beckFirstOrderMethodsOptimization2017} gives for any $\tau \leq \frac{1}{L}$, and any $x, \tilde{x} \in \mathbb{R}^d$:
    \begin{equation}
        g(x) \leq g(\tilde{x}) + \langle x - \tilde{x}, \nabla g (\tilde{x}) \rangle + \frac{1}{2 \tau} \| x - \tilde{x}\|^2.
        \label{eq:descent-lemma}
    \end{equation}
    Now since $F$ is a tangent majorant of $f$ we also have for all $x, \tilde{x} \in \mathbb{R}^d$ that:
    \begin{equation}
        f(x) \leq F(x, \tilde{x}).
        \label{eq:annex1}
    \end{equation}
    By adding equations (\ref{eq:descent-lemma}) and (\ref{eq:annex1}) we obtain: 
   \begin{equation*}
       h(x) \leq F(x, \tilde{x}) +  g(\tilde{x}) + \langle x - \tilde{x}, \nabla g (\tilde{x}) \rangle + \frac{1}{2 \tau} \| x - \tilde{x}\|^2
       \leq H(x, \tilde{x}).
   \end{equation*}
       
   The minimum of $H(\cdot, \tilde{x})$ can be expressed in terms of a proximal operator:
   \begin{align*}
        & \argmin_x \left\{ F(x, \tilde{x}) +  g(\tilde{x}) + \langle x - \tilde{x}, \nabla g (\tilde{x}) \rangle + \frac{1}{2 \tau} \| x - \tilde{x}\|^2 \right\}
     \\ & =\argmin_x \left\{ F(x, \tilde{x}) + \frac{1}{2 \tau} \| x - (\tilde{x} - \tau \nabla g(\tilde{x}))\|^2 \right\}
     \\ & =\text{prox}_{\tau F(., \tilde{x})} \circ (\text{Id} - \tau \nabla g) (\tilde{x}).
   \end{align*}
    And so by setting $x^{(n+1)} = \text{prox}_{\tau F(., x^{(n)})} \circ (\text{Id} - \tau \nabla g) (x^{(n)})$ and using Lemma \ref{lem:mm} we obtain that the iterates $\left( h(x^{{(n)}}) \right)$ are non-increasing. Since $h$ is lower-bounded,  $\left( h(x^{{(n)}}) \right)$ converges to a limit $h^*$.
     
    2) Given that $H(x^{(n+1)}, x^{(n)}) \leq H(x^{(n)}, x^{(n)})$, by definition of $x^{(n+1)}$ we have:
    \begin{align*}
        F(x^{(n+1)}, x^{{(n)}}) &\leq f(x^{{(n)}}) - \langle x^{(n+1)} - x^{{(n)}}, \nabla g ( x^{(n+1)}) \rangle 
        \\ &~~ - \frac{1}{2 \tau} \| x^{(n+1)} - x^{{(n)}} \|^2.
    \end{align*}
    Since $F$ is a tangent majorant of $f$ we also have $f(x^{(n+1)}) \leq F(x^{(n+1)}, x^{{(n)}})$, thus:
    \begin{align*}
        f(x^{(n+1)}) & \leq f(x^{{(n)}}) - \langle x^{(n+1)} - x^{{(n)}}, \nabla g ( x^{{(n)}}) \rangle \\\ & ~~ - \frac{1}{2 \tau} \| x^{(n+1)} - x^{{(n)}} \|^2.
    \end{align*}
    We can now simply follow the usual proof of the convergence rate of the proximal gradient descent, using the descent lemma (\ref{eq:descent-lemma}) with a step size $\frac{1}{L}$.
    \begin{align*}
        h(x^{n+1}) &= f(x^{(n+1)}) + g(x^{(n+1)})
                \\ &\leq f(x^{{(n)}}) - \langle x^{(n+1)} - x^{{(n)}}, \nabla g ( x^{{(n)}}) \rangle - \frac{1}{2 \tau} \| x^{(n+1)} - x^{{(n)}} \|^2 
                \\ &+ g(x^{{(n)}}) + \langle x^{(n+1)} - x^{{(n)}}, \nabla g (x^{{(n)}}) \rangle + \frac{L}{2} \| x^{(n+1)} - x^{{(n)}}\|^2
                \\ &\leq h(x^{{(n)}}) - \left(\frac{1}{2\tau} - \frac{L}{2}\right)  \| x^{(n+1)} - x^{{(n)}}\|^2.
    \end{align*}
    By slightly rearranging the terms we end up with:
    \begin{equation*}
        \| x^{(n+1)} - x^{{(n)}}\|^2 \leq \frac{1}{\frac{1}{2\tau} - \frac{L}{2}} \left( h(x^{{(n)}}) - h(x^{(n+1)}) \right).
    \end{equation*}
    Now if we sum over all the iterates up until iterate $N$ we have:
    \begin{align*}
        \sum_{n}^{N} \| x^{(n+1)} - x^{{(n)}} \|^2 &\leq \frac{1}{\frac{1}{2\tau} - \frac{L}{2}} \left( h(x^{(0)}) - h(x^{(N)}) \right)  
        \\ &\leq \frac{1}{\frac{1}{2\tau} - \frac{L}{2}} \left( h(x^{(0)}) - h^{*}\right)
    \end{align*}
    and so if we define $\gamma_N = \min_{0 \leq n \leq N} \|x^{(n+1)} - x^{{(n)}}\|^2$ we have:
    \begin{equation*}
        \gamma_N \leq \frac{1}{N} \frac{1}{\frac{1}{2\tau} - \frac{L}{2}} \left( h(x^{(0)}) - h^{*}\right)
    \end{equation*}
    and thus $\lim_{N \to \infty} \gamma_N = 0$ with rate $\mathcal{O}(1/N)$.
    
    3) Since for all $\tilde{x} \in \mathbb{R}^d$, $F(\cdot, \tilde{x})$ is convex with respect to its first argument, the proof used in the standard proximal gradient descent remains valid. In particular, one can directly apply the proof provided in \cite{huraultGradientStepDenoiser2021}[Theorem 1 iii)] to establish this result.
    
\subsection{Computation of $\mathrm{prox}_{\tau F_{\text{EM}}}$}
\label{annex:prox-computation}
The proximal operator of $F_{\text{EM}}$ with respect to its first argument is defined as:
    \begin{equation*}
        \text{prox}_{\tau F_{\text{EM}}(.,\tilde{x})} (u) = \argmin_{x\ge 0} F_{EM}(x, \tilde{x}) + \frac{1}{2 \tau}\| x - u\|^2 = \argmin_{x\ge 0} Q(x, u).
    \end{equation*}
Again the proof is straightforward by computing the gradient of this expression: 
\begin{equation}
     \frac{\partial Q(x,u) }{\partial x_k} = \left( s_k - \sum_i \frac{y_i \tilde{\alpha}_{ik}}{x_k} \right) + \frac{x_k - u_k}{\tau}
\end{equation}
then solving the quadratic equation given by the KKT conditions $x_k\frac{\partial Q(x, u)}{\partial x_k} = 0$:
\begin{equation}
    x_k^2 + \tau \left(s_k - \frac{u_k}{\tau} \right) x_k - \tau \sum_i y_i \tilde{\alpha}_{ik} = 0
\end{equation}
This equation has one solution under the constraint $x_k \geq 0$:
\begin{equation*}
    x_k = \frac{1}{2} \left[ u_k - \tau s_k + \sqrt{\left( u_k - \tau s_k \right)^2 + 4 \tau \sum_i y_i \tilde{\alpha}_{ik}} \right]
\end{equation*}

\subsection{Additional Experiments}
\label{annex:additional-exp}

In this section we present additional quantitative results for our different experiments. Namely, in Table \ref{annex:tab:deblur-cbsd68-quantitative} we report deblurring results on the CBSD68 validation dataset. In Table \ref{annex:tab:pet-denoising} we report the denoising performances of the GS-DRUNet network on the Brainweb dataset. 

\onecolumn

\begin{table*}[!h]
\caption{Comparison on the CBSD68 validation dataset}
\label{annex:tab:deblur-cbsd68-quantitative}
\resizebox{\textwidth}{!}{%
\begin{tabular}{lcccccccccccccccc}
\multicolumn{1}{c}{}          &                                                                             & \multicolumn{3}{c}{\includegraphics[width=0.1\linewidth]{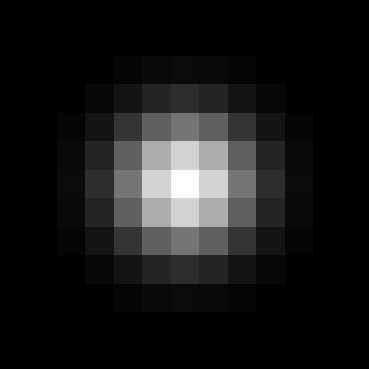}} & \multicolumn{3}{c}{\includegraphics[width=0.1\linewidth]{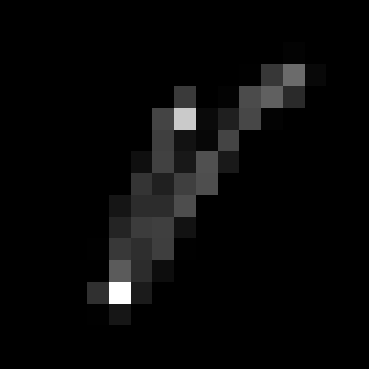}} & \multicolumn{3}{c}{\includegraphics[width=0.1\linewidth]{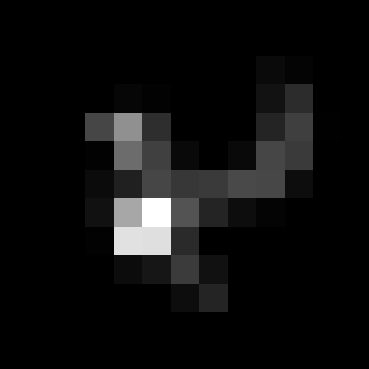}} & \multicolumn{3}{c}{\includegraphics[width=0.1\linewidth]{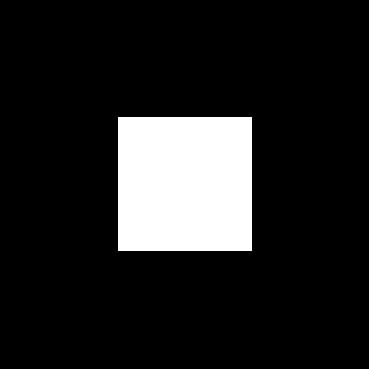}} & \multicolumn{3}{c}{Average}                            \\ \cline{1-14}
\multicolumn{1}{|c|}{$\zeta$} & \multicolumn{1}{c|}{Methods}                                                & PSNR $\uparrow$      & SSIM $\uparrow$     & \multicolumn{1}{c|}{LPIPS $\downarrow$}     & PSNR $\uparrow$      & SSIM $\uparrow$      & \multicolumn{1}{c|}{LPIPS $\downarrow$}     & PSNR $\uparrow$      & SSIM $\uparrow$      & \multicolumn{1}{c|}{LPIPS $\downarrow$}     & PSNR $\uparrow$     & SSIM $\uparrow$     & \multicolumn{1}{c|}{LPIPS $\downarrow$}     & PSNR $\uparrow$ & SSIM $\uparrow$ & LPIPS $\downarrow$ \\ \cline{1-14}
\multicolumn{1}{|l|}{60}      & \multicolumn{1}{c|}{RL-TV \cite{maximTomographicReconstructionPoisson2023}} & 24.02                & 0.608               & \multicolumn{1}{c|}{0.424}                  & 22.59                & 0.516                & \multicolumn{1}{c|}{0.485}                  & 23.35                & 0.553                & \multicolumn{1}{c|}{0.429}                  & 22.93               & 0.549               & \multicolumn{1}{c|}{{\ul 0.488}}            & 23.22           & 0.556           & 0.457              \\
\multicolumn{1}{|l|}{}        & \multicolumn{1}{c|}{DPS \cite{chungDiffusionPosteriorSampling2022}}         & 18.78                & 0.436               & \multicolumn{1}{c|}{0.574}                  & 21.17                & 0.519                & \multicolumn{1}{c|}{0.392}                  & 21.94                & 0.567                & \multicolumn{1}{c|}{0.391}                  & 18.24               & 0.374               & \multicolumn{1}{c|}{0.559}                  & 20.03           & 0.474           & 0.479              \\
\multicolumn{1}{|l|}{}        & \multicolumn{1}{c|}{Prox-DiffPIR \cite{melidonis2025scorebased}}            & 24.98                & 0.682               & \multicolumn{1}{c|}{\textbf{0.374}}         & 24.10                & 0.615                & \multicolumn{1}{c|}{\textbf{0.326}}         & 25.18                & 0.675                & \multicolumn{1}{c|}{\textbf{0.275}}         & {\ul 23.73}         & {\ul 0.603}         & \multicolumn{1}{c|}{\textbf{0.449}}         & 24.50           & 0.644           & \textbf{0.356}     \\
\multicolumn{1}{|l|}{}        & \multicolumn{1}{c|}{PhD-Net \cite{sanghviPhotonLimitedNonBlind2022}}        & 24.98                & 0.670               & \multicolumn{1}{c|}{0.465}                  & 23.52                & 0.588                & \multicolumn{1}{c|}{0.515}                  & 24.75                & 0.647                & \multicolumn{1}{c|}{0.472}                  & 23.65               & 0.595               & \multicolumn{1}{c|}{0.559}                  & 24.23           & 0.625           & 0.503              \\
\multicolumn{1}{|l|}{}        & \multicolumn{1}{c|}{DPIR \cite{zhangPlugPlayImageRestoration2022}}          & {\ul 25.39}          & {\ul 0.690}         & \multicolumn{1}{c|}{0.442}                  & {\ul 24.53}          & 0.635                & \multicolumn{1}{c|}{0.462}                  & 25.56                & 0.686                & \multicolumn{1}{c|}{0.394}                  & 23.68               & 0.595               & \multicolumn{1}{c|}{0.565}                  & 24.79           & 0.652           & 0.466              \\
\multicolumn{1}{|l|}{}        & \multicolumn{1}{c|}{B-PnP \cite{huraultConvergentBregmanPlugandplay2024a}}  & 25.26                & 0.686               & \multicolumn{1}{c|}{0.463}                  & 24.42                & {\ul 0.638}          & \multicolumn{1}{c|}{0.463}                  & {\ul 25.90}          & {\ul 0.707}          & \multicolumn{1}{c|}{0.398}                  & 23.70               & 0.599               & \multicolumn{1}{c|}{0.562}                  & {\ul 24.82}     & {\ul 0.657}     & 0.472              \\
\multicolumn{1}{|l|}{}        & \multicolumn{1}{c|}{\textbf{PnP-MM (Ours)}}                                 & \textbf{25.63}       & \textbf{0.703}      & \multicolumn{1}{c|}{{\ul 0.409}}            & \textbf{25.12}       & \textbf{0.672}       & \multicolumn{1}{c|}{{\ul 0.388}}            & \textbf{26.47}       & \textbf{0.738}       & \multicolumn{1}{c|}{{\ul 0.311}}            & \textbf{24.13}      & \textbf{0.617}      & \multicolumn{1}{c|}{0.516}                  & \textbf{25.34}  & \textbf{0.683}  & {\ul 0.406}        \\ \cline{1-14}
\multicolumn{1}{|l|}{40}      & \multicolumn{1}{c|}{RL-TV \cite{maximTomographicReconstructionPoisson2023}} & 23.81                & 0.599               & \multicolumn{1}{c|}{0.441}                  & 22.67                & 0.530                & \multicolumn{1}{c|}{0.483}                  & 23.51                & 0.571                & \multicolumn{1}{c|}{0.442}                  & 22.70               & 0.539               & \multicolumn{1}{c|}{{\ul 0.506}}            & 23.17           & 0.560           & 0.468              \\
\multicolumn{1}{|l|}{}        & \multicolumn{1}{c|}{DPS \cite{chungDiffusionPosteriorSampling2022}}         & 18.82                & 0.431               & \multicolumn{1}{c|}{0.586}                  & 20.73                & 0.507                & \multicolumn{1}{c|}{0.394}                  & 21.50                & 0.549                & \multicolumn{1}{c|}{0.404}                  & 18.09               & 0.381               & \multicolumn{1}{c|}{0.578}                  & 19.79           & 0.467           & 0.491              \\
\multicolumn{1}{|l|}{}        & \multicolumn{1}{c|}{Prox-DiffPIR \cite{melidonis2025scorebased}}            & 24.63                & 0.663               & \multicolumn{1}{c|}{\textbf{0.390}}         & 23.73                & 0.599                & \multicolumn{1}{c|}{\textbf{0.370}}         & 24.73                & 0.652                & \multicolumn{1}{c|}{\textbf{0.312}}         & 23.42               & 0.587               & \multicolumn{1}{c|}{\textbf{0.471}}         & 24.13           & 0.625           & \textbf{0.386}     \\
\multicolumn{1}{|l|}{}        & \multicolumn{1}{c|}{PhD-Net \cite{sanghviPhotonLimitedNonBlind2022}}        & 24.66                & 0.655               & \multicolumn{1}{c|}{0.486}                  & 23.23                & 0.573                & \multicolumn{1}{c|}{0.543}                  & 24.35                & 0.626                & \multicolumn{1}{c|}{0.543}                  & 23.41               & 0.584               & \multicolumn{1}{c|}{0.573}                  & 23.91           & 0.609           & 0.524              \\
\multicolumn{1}{|l|}{}        & \multicolumn{1}{c|}{DPIR \cite{zhangPlugPlayImageRestoration2022}}          & {\ul 25.09}          & 0.674               & \multicolumn{1}{c|}{0.458}                  & 24.04                & 0.612                & \multicolumn{1}{c|}{0.491}                  & 24.84                & 0.660                & \multicolumn{1}{c|}{0.423}                  & 23.42               & 0.583               & \multicolumn{1}{c|}{0.582}                  & 24.35           & 0.632           & 0.489              \\
\multicolumn{1}{|l|}{}        & \multicolumn{1}{c|}{B-PnP \cite{huraultConvergentBregmanPlugandplay2024a}}  & 24.97                & {\ul 0.678}         & \multicolumn{1}{c|}{0.484}                  & {\ul 24.18}          & {\ul 0.631}          & \multicolumn{1}{c|}{0.462}                  & {\ul 25.50}          & {\ul 0.697}          & \multicolumn{1}{c|}{0.410}                  & {\ul 23.56}         & {\ul 0.595}         & \multicolumn{1}{c|}{0.571}                  & {\ul 24.55}     & {\ul 0.650}     & 0.481              \\
\multicolumn{1}{|l|}{}        & \multicolumn{1}{c|}{\textbf{PnP-MM (Ours)}}                                 & \textbf{25.37}       & \textbf{0.691}      & \multicolumn{1}{c|}{{\ul 0.414}}            & \textbf{24.71}       & \textbf{0.650}       & \multicolumn{1}{c|}{{\ul 0.388}}            & \textbf{25.99}       & \textbf{0.715}       & \multicolumn{1}{c|}{{\ul 0.319}}            & \textbf{23.92}      & \textbf{0.608}      & \multicolumn{1}{c|}{0.516}                  & \textbf{25.00}  & \textbf{0.666}  & {\ul 0.411}        \\ \cline{1-14}
\multicolumn{1}{|l|}{20}      & \multicolumn{1}{c|}{RL-TV \cite{maximTomographicReconstructionPoisson2023}} & 23.29                & 0.575               & \multicolumn{1}{c|}{0.477}                  & 22.29                & 0.521                & \multicolumn{1}{c|}{0.508}                  & 23.20                & 0.565                & \multicolumn{1}{c|}{0.469}                  & 22.24               & 0.518               & \multicolumn{1}{c|}{{\ul 0.538}}            & 22.76           & 0.545           & 0.498              \\
\multicolumn{1}{|l|}{}        & \multicolumn{1}{c|}{DPS \cite{chungDiffusionPosteriorSampling2022}}         & 17.58                & 0.382               & \multicolumn{1}{c|}{0.631}                  & 20.49                & 0.494                & \multicolumn{1}{c|}{\textbf{0.398}}         & 21.60                & 0.541                & \multicolumn{1}{c|}{0.395}                  & 17.78               & 0.349               & \multicolumn{1}{c|}{0.575}                  & 19.36           & 0.442           & 0.500              \\
\multicolumn{1}{|l|}{}        & \multicolumn{1}{c|}{Prox-DiffPIR \cite{melidonis2025scorebased}}            & 24.00                & 0.629               & \multicolumn{1}{c|}{\textbf{0.423}}         & 23.06                & 0.567                & \multicolumn{1}{c|}{{\ul 0.441}}            & 23.98                & 0.616                & \multicolumn{1}{c|}{\textbf{0.378}}         & 22.84               & 0.555               & \multicolumn{1}{c|}{\textbf{0.518}}         & 23.47           & 0.592           & \textbf{0.440}     \\
\multicolumn{1}{|l|}{}        & \multicolumn{1}{c|}{PhD-Net \cite{sanghviPhotonLimitedNonBlind2022}}        & 24.06                & 0.627               & \multicolumn{1}{c|}{0.518}                  & 22.81                & 0.554                & \multicolumn{1}{c|}{0.577}                  & 23.76                & 0.599                & \multicolumn{1}{c|}{0.533}                  & 22.94               & 0.563               & \multicolumn{1}{c|}{0.597}                  & 23.39           & 0.585           & 0.556              \\
\multicolumn{1}{|l|}{}        & \multicolumn{1}{c|}{DPIR \cite{zhangPlugPlayImageRestoration2022}}          & {\ul 24.42}          & {\ul 0.637}         & \multicolumn{1}{c|}{0.498}                  & {\ul 23.36}          & 0.581                & \multicolumn{1}{c|}{0.553}                  & {\ul 24.60}          & {\ul 0.638}          & \multicolumn{1}{c|}{0.472}                  & 22.92               & 0.558               & \multicolumn{1}{c|}{0.617}                  & {\ul 23.82}     & {\ul 0.603}     & 0.535              \\
\multicolumn{1}{|l|}{}        & \multicolumn{1}{c|}{B-PnP \cite{huraultConvergentBregmanPlugandplay2024a}}  & 23.43                & 0.598               & \multicolumn{1}{c|}{0.627}                  & 23.12                & {\ul 0.588}          & \multicolumn{1}{c|}{0.506}                  & 23.09                & 0.583                & \multicolumn{1}{c|}{0.570}                  & {\ul 22.97}         & {\ul 0.572}         & \multicolumn{1}{c|}{0.627}                  & 23.15           & 0.585           & 0.582              \\
\multicolumn{1}{|l|}{}        & \multicolumn{1}{c|}{\textbf{PnP-MM (Ours)}}                                 & \textbf{24.78}       & \textbf{0.663}      & \multicolumn{1}{c|}{{\ul 0.448}}            & \textbf{23.94}       & \textbf{0.614}       & \multicolumn{1}{c|}{0.459}                  & \textbf{25.11}       & \textbf{0.676}       & \multicolumn{1}{c|}{{\ul 0.391}}            & \textbf{23.43}      & \textbf{0.585}      & \multicolumn{1}{c|}{0.557}                  & \textbf{24.32}  & \textbf{0.635}  & \textbf{0.464}     \\ \cline{1-14}
\multicolumn{1}{|l|}{5}       & \multicolumn{1}{c|}{RL-TV \cite{maximTomographicReconstructionPoisson2023}} & 21.45                & 0.469               & \multicolumn{1}{c|}{0.586}                  & 20.81                & 0.432                & \multicolumn{1}{c|}{0.585}                  & 20.86                & 0.428                & \multicolumn{1}{c|}{0.602}                  & 21.21               & 0.469               & \multicolumn{1}{c|}{{\ul 0.577}}            & 21.08           & 0.449           & 0.596              \\
\multicolumn{1}{|l|}{}        & \multicolumn{1}{c|}{DPS \cite{chungDiffusionPosteriorSampling2022}}         & 19.11                & 0.444               & \multicolumn{1}{c|}{\textbf{0.495}}         & 17.89                & 0.375                & \multicolumn{1}{c|}{\textbf{0.525}}         & 18.39                & 0.414                & \multicolumn{1}{c|}{0.520}                  & 17.80               & 0.383               & \multicolumn{1}{c|}{\textbf{0.516}}         & 18.29           & 0.404           & \textbf{0.514}     \\
\multicolumn{1}{|l|}{}        & \multicolumn{1}{c|}{Prox-DiffPIR \cite{melidonis2025scorebased}}            & 22.65                & 0.556               & \multicolumn{1}{c|}{{\ul 0.503}}            & {\ul 21.85}          & 0.509                & \multicolumn{1}{c|}{{\ul 0.538}}            & {\ul 22.54}          & 0.546                & \multicolumn{1}{c|}{{\ul 0.497}}            & 21.85               & 0.505               & \multicolumn{1}{c|}{0.595}                  & {\ul 22.22}     & 0.529           & {\ul 0.533}        \\
\multicolumn{1}{|l|}{}        & \multicolumn{1}{c|}{PhD-Net \cite{sanghviPhotonLimitedNonBlind2022}}        & {\ul 22.71}          & 0.566               & \multicolumn{1}{c|}{0.579}                  & 21.73                & 0.508                & \multicolumn{1}{c|}{0.631}                  & 22.40                & 0.539                & \multicolumn{1}{c|}{0.596}                  & {\ul 21.88}         & {\ul 0.521}         & \multicolumn{1}{c|}{0.647}                  & 22.18           & 0.534           & 0.647              \\
\multicolumn{1}{|l|}{}        & \multicolumn{1}{c|}{DPIR \cite{zhangPlugPlayImageRestoration2022}}          & 17.34                & 0.355               & \multicolumn{1}{c|}{0.760}                  & 12.02                & 0.272                & \multicolumn{1}{c|}{0.756}                  & 10.68                & 0.230                & \multicolumn{1}{c|}{0.784}                  & 19.70               & 0.408               & \multicolumn{1}{c|}{0.787}                  & 14.93           & 0.316           & 0.772              \\
\multicolumn{1}{|l|}{}        & \multicolumn{1}{c|}{B-PnP \cite{huraultConvergentBregmanPlugandplay2024a}}  & 21.33                & {\ul 0.572}         & \multicolumn{1}{c|}{0.580}                  & 20.64                & {\ul 0.521}          & \multicolumn{1}{c|}{0.613}                  & 21.31                & {\ul 0.568}          & \multicolumn{1}{c|}{0.566}                  & 20.62               & 0.515               & \multicolumn{1}{c|}{0.655}                  & 20.97           & {\ul 0.544}     & 0.603              \\
\multicolumn{1}{|l|}{}        & \multicolumn{1}{c|}{\textbf{PnP-MM (Ours)}}                                 & \textbf{23.50}       & \textbf{0.601}      & \multicolumn{1}{c|}{0.513}                  & \textbf{22.56}       & \textbf{0.550}       & \multicolumn{1}{c|}{0.550}                  & \textbf{23.48}       & \textbf{0.598}       & \multicolumn{1}{c|}{\textbf{0.493}}         & \textbf{22.49}      & \textbf{0.541}      & \multicolumn{1}{c|}{0.607}                  & \textbf{23.01}  & \textbf{0.573}  & 0.541              \\ \cline{1-14}
\end{tabular}
}
\end{table*}

\begin{table}[h]
\caption{Quantitative denoising performances on the Brainweb test dataset}
\label{annex:tab:pet-denoising}
\centering
\begin{tabular}{|l|ccccccc|}
    \hline
    \textbf{Noise level ($\sigma$)} & 1                & 5                & 10               & 20              & 30              & 40              & 50              \\ \hline
    \textbf{NRMSE (\%)}             & $2.18 \pm 0.302$ & $3.68 \pm 0.414$ & $6.08 \pm 0.691$ & $10.4 \pm 1.52$ & $13.8 \pm 1.96$ & $16.2 \pm 2.27$ & $18.6 \pm 2.94$ \\ \hline
\end{tabular}
\end{table}

\twocolumn

\subsection{Hyper-parameters}
\label{annex:hyperparams}

In this section, we detail the hyperparameters used in the experiments of Section~\ref{sec:exp}. Specifically, Table \ref{annex:tab:deblur-hyperparam} reports the settings for the deconvolution experiment, Table \ref{annex:tab:pet-hyperparam} for the PET experiment, and Table \ref{annex:tab:ct-hyperparam} for the CT experiment. For the deblurring experiment, for B-PnP \cite{huraultConvergentBregmanPlugandplay2024a} and Prox-DiffPIR \cite{melidonis2025scorebased} we have re-used the hyper-parameters selected in their experiments. Finally, for DPS \cite{chungDiffusionPosteriorSampling2022}, we used the hyper-parameters chosen by Melidionis et \textit{al.} in their comparison in \cite{melidonis2025scorebased}.

\begin{table*}[!t]
\centering
\label{annex:tab:hyperparam}

\begin{minipage}[t]{0.48\linewidth}
\centering
\caption{Deblurring experiments.}
\label{annex:tab:deblur-hyperparam}
\resizebox{\linewidth}{!}{%
\begin{tabular}{|l|l|c|c|}
\hline
Noise level $\zeta$ & RL-TV          & PnP-MM              & DPIR                \\ \hline
\multirow{3}{*}{60} & $\lambda=0.03$ & $\lambda=0.86$      & -                   \\
                    & -              & $\sigma = 16 / 255$ & $\sigma = 25 / 255$ \\
                    & -              & $\tau=0.3$          & -                   \\ \hline
\multirow{3}{*}{40} & $\lambda=0.05$ & $\lambda=0.80$      & -                   \\
                    & -              & $\sigma = 22 / 255$ & $\sigma = 30 / 255$ \\
                    & -              & $\tau=0.16$         & -                   \\ \hline
\multirow{3}{*}{20} & $\lambda=0.12$ & $\lambda=0.90$      & -                   \\
                    & -              & $\sigma = 28 / 255$ & $\sigma = 46 / 255$ \\
                    & -              & $\tau=0.26$         & -                   \\ \hline
\multirow{3}{*}{5}  & $\lambda=0.14$ & $\lambda=0.95$      & -                   \\
                    & -              & $\sigma = 56 / 255$ & $\sigma = 50 / 255$ \\
                    & -              & $\tau=0.33$         & -                   \\ \hline
\end{tabular}}
\end{minipage}
\hfill
\begin{minipage}[t]{0.48\linewidth}
\centering
\caption{PET reconstruction experiment.}
\label{annex:tab:pet-hyperparam}
\resizebox{\linewidth}{!}{%
\begin{tabular}{|l|ccccc|}
\hline
Method           & $\sigma$ & $\lambda$ & $\tau$ & Subsets & Iterations \\ \hline
PnP-MM           & 5        & 0.3       & 100    & 14      & 50        \\ \hline
OSEM             & --       & --        & --     & 14      & 10        \\ \hline
mr-MAP-EM        & --       & 0.06      & --     & 14      & 10        \\ \hline
FBSEM / mr-FBSEM & --       & --        & --     & 6       & 10        \\ \hline
\end{tabular}}
\end{minipage}

\end{table*}

\begin{table*}[h]
\centering
\caption{Hyperparameters for the CT reconstruction experiment.}
\label{annex:tab:ct-hyperparam}
\resizebox{0.8\linewidth}{!}{%
\begin{tabular}{|l|l|ccc|ccc|ccc|}
\hline
                        &                     & \multicolumn{3}{c|}{192 angles}              & \multicolumn{3}{c|}{128 angles}              & \multicolumn{3}{c|}{64 angles}               \\ \hline
Method                  & Parameter           & $\sigma=0.5\%$ & $\sigma=1\%$ & $\sigma=2\%$ & $\sigma=0.5\%$ & $\sigma=1\%$ & $\sigma=2\%$ & $\sigma=0.5\%$ & $\sigma=1\%$ & $\sigma=2\%$ \\ \hline
MLEM                    & Iterations           & 40             & 35           & 20           & 40             & 35           & 20           & 20             & 20           & 20           \\ \hline
\multirow{4}{*}{PnP-MM} & $\lambda$           & 0.12           & 0.20         & 0.20         & 0.10           & 0.15         & 0.15         & 0.10           & 0.10         & 0.10         \\
                        & $\sigma \times 255$ & 4              & 4            & 8            & 5              & 5            & 8            & 5              & 8            & 8            \\
                        & $\tau$              & 1              & 1            & 1            & 1              & 1            & 1            & 1              & 1            & 1            \\
                        & Iterations          & 600            & 600          & 600          & 600            & 600          & 600          & 600            & 600          & 600          \\ \hline
\end{tabular}}
\end{table*}

\end{document}